\documentclass[journal]{IEEEtai}

\usepackage[colorlinks,urlcolor=blue,linkcolor=blue,citecolor=blue]{hyperref}

\usepackage{color,array}
\usepackage{booktabs}
\usepackage{amsmath}
\usepackage{amsthm}
\usepackage{amssymb}
\usepackage{comment}

\usepackage{graphicx}

\newtheorem{theorem}{Theorem}

\newtheorem{corollary}[theorem]{Corollary}
\newtheorem{proposition}[theorem]{Proposition}

\begin{document}

\title{Generalized Holographic Reduced Representations} 

\author{%
  Calvin Yeung,
  Zhuowen Zou,
  SungHeon Jeong,
  Wenjun Huang,
  Nathaniel D. Bastian,
  and Mohsen Imani%
  \thanks{C. Yeung, Z. Zou, S. Jeong, and M. Imani are with the University of California, Irvine, CA 92697 USA (e-mail: chyeung2@uci.edu; zhuowez1@uci.edu; sungheoj@uci.edu; m.imani@uci.edu).}%
  \thanks{N. D. Bastian is with the United States Military Academy, West Point, NY 10996 USA (e-mail: nathaniel.bastian@westpoint.edu).}%
}

\markboth{Journal of IEEE Transactions on Artificial Intelligence, Vol. 00, No. 0, Month 2020}
{Yeung \MakeLowercase{\textit{et al.}}: Generalized Holographic Reduced Representations}

\maketitle

\begin{abstract}

Hyperdimensional Computing (HDC) is a computationally and data-efficient paradigm that acts as a bridge between connectionist and symbolic approaches to artificial intelligence (AI). However, HDC's simplicity poses challenges for encoding complex compositional structures, especially in its binding operation. To address this, we propose Generalized Holographic Reduced Representations (GHRR), an extension of Fourier Holographic Reduced Representations (FHRR), a specific HDC implementation. GHRR introduces a flexible, non-commutative binding operation, enabling improved encoding of complex data structures while preserving HDC's desirable properties of robustness and transparency. In this work, we introduce the GHRR framework, prove its theoretical properties and its adherence to HDC properties, explore its kernel and binding characteristics, and perform empirical experiments showcasing its flexible non-commutativity, enhanced decoding accuracy for compositional structures. We also demonstrate that binding in GHRR is more expressive than that in other HDC variants; in particular, we show that binding in GHRR can implement a kind of attention mechanism. We verify this by replacing the attention mechanism in a transformer with its GHRR-equivalent and testing it on a language modeling task, showing improved performance compared to a vanilla transformer.

\end{abstract}

\begin{IEEEImpStatement}
This work introduces Generalized Holographic Reduced Representations (GHRR), a novel extension of hyperdimensional computing that bridges symbolic and connectionist AI by providing flexible, non-commutative binding operations. GHRR preserves the robustness, interpretability, and fixed-width nature of traditional Fourier Holographic Reduced Representations while enabling accurate encoding and decoding of complex compositional structures. We show that GHRR binding can implement attention mechanisms, and by replacing the transformer's standard attention with a GHRR-based equivalent, achieve improved language modeling performance. These advances lower computational and data requirements for structured representation learning, promote transparent neuro-symbolic integration, and open new pathways for efficient, interpretable AI systems.
\end{IEEEImpStatement}


\section{Introduction}
In the past decade of artificial intelligence research, deep learning has seen monumental success, finding applications in domains such as classification, image generation, and language modeling \cite{krizhevskyImageNetClassificationDeep2012, goodfellowGenerativeAdversarialNets2014, hoDenoisingDiffusionProbabilistic2020, vaswaniAttentionAllYou2017a, brownLanguageModelsAre2020}. Central to the success of deep learning is its ability to learn representations that preserve task-relevant structure in the data, a necessary component for models with generalization capabilities. Indeed, this reliance on the quality of representations is reflected in the common use of pre-trained models, and, more recently, in foundation models \cite{Bommasani2021FoundationModels}. 

However, underpinning deep learning models with good representations are tremendous computational resources and internet-scale data in the order of terabytes and beyond. While research into scaling laws suggests that it is possible to further improve the model and thus representation quality by continuing to scale up model size, data quantity, and computational resources \cite{kaplanScalingLawsNeural2020}, this approach is extremely cost-inefficient, time-consuming, and excludes all but the largest and most resourceful organizations from the development of such models. Thus, it is necessary to explore directions that are \textit{compute- and data-efficient} while having representational properties amenable to generalization. 

In recent years, Hyperdimensional Computing (HDC), or Vector Symbolic Architectures (VSA), has emerged as a neurosymbolic computational paradigm in artificial intelligence bridging the gap between connectionist and symbolic frameworks \cite{kanervaHyperdimensionalComputingIntroduction2009,kleykoSurveyHyperdimensionalComputing2023,renkhoff2024survey,vergesClassificationUsingHyperdimensional2025}. In this way, HDC can act as an interface that enables the explicit specification of structure as in symbolic approaches while maintaining the flexibility and power of connectionist, especially deep learning, approaches. HDC forms an algebra over higher-dimensional vectors, called hypervectors, with operations corresponding to symbolic operations. For example, the bundling operation ($+$) corresponds to disjunction and binding ($*$) corresponds to conjunction of symbols. Through the use of such operations, it is possible to build up complex representations and data structures \cite{kleykoAutoscalingBloomFilter2020,kleykoVectorSymbolicArchitectures2022,poduval2022graphd}, and consequently instill a factorized structure onto the space of representations, which can subsequently be used in downstream tasks such as symbolic reasoning \cite{hersche2023neuro,herscheProbabilisticAbductionVisual2023}. 

Recent research has highlighted the significant benefits of HDC in improving the computational efficiency and symbolic reasoning capabilities of neural networks~\cite{menet2024mimonets,hersche2023neuro,liuHyperDynDynamicDimensional2025}.

While HDC shows promise as a neuro-vector-symbolic framework for factorized and compositional representations~\cite{rachkovskijShiftEquivariantSimilarityPreservingHypervector2024}, its simplicity poses a problem for its expressive power and its ability to encode more complex structures. In particular, crucial to the HDC framework is the binding operation, which enables the formation of representational complexes from simpler parts, often interpreted as an association or conjunction of concepts. This, however, is limiting, as concepts may hold various kinds of relationships with each other and may be combined in multiple ways. In addition, implementations of binding in HDC are typically commutative, so additional techniques such as applying permutations or positional encodings are often needed to encode structure, which complicates the representation of data structures. Thus, we need a more flexible and non-commutative binding operation more suitable for encoding complex structures, while still maintaining the other desirable properties of HDC, such as transparency, robustness to noise, interpretability, and fixed-width composition.

The design of our proposed framework stems from recognizing (1) the importance of high-dimensional, holographic, fixed-width representations in maintaining desirable properties of HDC; (2) the lack of expressivity of the binding operation in current HDC implementations; and (3) its close connection to effective and efficient HDC learning~\cite{imani2020dual, hernandez2021onlinehd}, namely in leveraging the Random Fourier Features (RFF) encoding~\cite{rahimiRandomFeaturesLargeScale} to map data to hypervectors. For the second point, elements of HDC hypervectors are typically in a subgroup of $(\mathbb{C},*)$, which limits algebraic complexity when applying the binding operation. For the third point, the effectiveness of HDC in learning tasks results largely from encoding inspired by RFF, enabling the efficient approximation of shift-invariant kernels via taking the similarity between hypervectors. To be precise, Fractional Power Encoding (FPE) \cite{fradyComputingFunctionsUsing2022, plate1992holographic, plate1994distributed} with binding-based composition gives rise to the RFF encoding.

With the three identified directions in consideration, we introduce Generalized Holographic Reduced Representations (GHRR), an extension of Fourier Holographic Reduced Representations (FHRR) \cite{plateHolographicReducedRepresentations1995,plateHolographicReducedRepresentations2000}, a particular implementation of HDC. While elements in FHRR hypervectors are in the unitary group $\mathrm U(1)\subseteq \mathbb{C}$, they are, in contrast, in $\mathrm U(m)$ in GHRR. While maintaining the kernel properties of FHRR \cite{rahimiRandomFeaturesLargeScale} and satisfying the basic constraints of HDC representations, GHRR provides a framework for flexible, non-commutative binding as well as adaptive kernels while enabling improved encoding of complex data structures via better decoding abilities and memorization of bound hypervectors. Our work has several key contributions:

\begin{enumerate}
    \item We introduce the GHRR framework and outline a particular implementation of GHRR, including variations of increasing complexity (Section \ref{sec:ghrr-overview}). 
    \item We prove that our implementation of GHRR satisfies the basic properties of HDC, including quasi-orthogonality and the similarity preservation of the binding operation (Section \ref{subsec:ghrr-prop}-\ref{subsec:ghrr-quasi}).
    \item We explore the kernel and binding properties of GHRR and provide an interpretation of binding in GHRR as an interpolation between binding in FHRR and in Tensor Product Representations \cite{smolenskyTensorProductVariable1990} (Section \ref{sec:enc-data}).
    \item We demonstrate that binding in GHRR is more expressive than that in other HDC variants (Section \ref{subsec:binding}); in particular, we show that binding in GHRR can implement a kind of attention mechanism (Section \ref{subsec:binding-attention}).
    \item We perform empirical experiments on GHRR, demonstrating its flexible non-commutativity, increased decoding accuracy for compositional structures, e.g. trees and improved memorization capacity for bound hypervectors compared to FHRR (Section \ref{sec:non-commutativty}-\ref{subsec:capacity}). 
    \item We replace the attention mechanism in a transformer with a GHRR binding-based attention mechanism and test it on a language modeling task, showing improved performance over a vanilla transformer (Section \ref{subsec:token}).
\end{enumerate}

\section{Background}
\subsection{Hyperdimensional Computing}

Hyperdimensional Computing (HDC) is a computational framework for representing and handling data. HDC uses high-dimensional vectors where each element contributes to the overall concept representation~\cite{kanervaHyperdimensionalComputingIntroduction2009}. HDC representations are robust to noise and partial information loss by using high-dimensional and holographic vectors: by distributing the information ``evenly" (holographic) and redundantly (``high-dimensional"), the loss of some dimensions does not significantly degrade the information content~\cite{kanerva1996binary, kanervaHyperdimensionalComputingIntroduction2009}. Concretely, HDC achieves this property by using randomly generated hypervectors that give rise to consistent behavior at the algorithm level, using the property of the high-dimensional space. 

HDC addresses symbolic reasoning by including operations on high-dimensional vectors that allow for manipulations similar to logic operations in symbolic AI~\cite{plate2003holographic}. In particular, HDC achieves disjunctive and conjunctive operations through bundling and binding operations, respectively. A similarity function (e.g. cosine similarity) between hypervectors is used to facilitate pattern recognition and memory recall.

The fundamental unit in HDC is a high dimensional vector, also called a hypervector. A hypervector $\mathbf{H}$ lives in some hyperspace $\mathcal{H}$, e.g., $\mathbb{R}^D$ for $D$ large. The collection of hypervectors, along with some operators, forms an algebra over vectors. Generally, there are two types of hypervectors: (1) base hypervectors, which are generated stochastically, e.g., $\mathbf{H}\sim N(0,I_D)$; and (2) composite hypervectors, which are created by combining hypervectors via the operators of the algebra. These hypervectors can be compared via a similarity function $\delta(\mathbf{H}_1,\mathbf{H}_2)$. Generally, base hypervectors are generated such that they are quasi-orthogonal with respect to the similarity function. The three main operations in HDC, bundling, binding, and permutation, can be characterized by how they affect the similarity of hypervectors. We describe the three operations below:
\begin{enumerate}
    \item Bundling ($+$): Typically implemented as element-wise addition. If $\mathbf{H}=\mathbf{H}_1+\mathbf{H}_2$, then both $\mathbf{H}_1$ and $\mathbf{H}_2$ are similar to $\mathbf{H}$. This corresponds to forming a disjunctive representation of the corresponding concepts.
    \item Binding ($*$): Typically implemented as element-wise multiplication. If $\mathbf{H}=\mathbf{H}_1*\mathbf{H}_2$, then $\mathbf{H}$ is dissimilar to both $\mathbf{H}_1$ and $\mathbf{H}_2$. Binding also has the important property of similarity preservation in the sense that for some hypervector $\mathbf{H}_3$, $\delta(\mathbf{H}_3*\mathbf{H}_1,\mathbf{H}_3*\mathbf{H}_2)\simeq\delta(\mathbf{H}_1,\mathbf{H}_2)$. This corresponds to forming a conjunctive representation of the corresponding concepts.
    \item Permutation ($\rho$): Typically implemented as a rotation of vector elements. Generally, $\delta(\rho(\mathbf{H}),\mathbf{H})\simeq 0$. Permutation is usually used to encode order in sequences.
\end{enumerate}

These operators can be used to build up complex compositional representations for structured data such as trees \cite{fradyResonatorNetworksEfficient2020} and graphs \cite{poduval2022graphd}, which can subsequently be used for downstream tasks \cite{zakeri2025enabling}. For example, these representations can be operated on in a symbolic manner (e.g. via ``unbinding'' to query the data structure), or can also be integrated into neural networks to form neuro-vector-symbolic architectures \cite{hersche2023neuro}.

It is important to note that the description above of HDC is general; there are various specific realizations of HDC with the above properties. For our work, we focus on Fourier Holographic Reduced Representation, which we expand on in the next section.

\subsubsection{HDC in neural networks}

Recent research has integrated HDC with deep neural networks (DNNs) to address two primary limitations of traditional connectionist models: high computational costs during inference and the lack of structured symbolic reasoning.

\paragraph{Efficiency and Superposition}
To address computational inefficiency, approaches like MIMONets~\cite{menet2024mimonets} utilize the HDC property of superposition to process multiple inputs concurrently. MIMONets generalize standard layers (fully connected, convolutional, and transformer blocks) by binding inputs with high-dimensional keys to project them into quasi-orthogonal subspaces. This allows a single pass of the neural network to process a bundled representation of multiple inputs. 

Empirical evaluations on standard benchmarks, such as CIFAR-10/100 and the Long Range Arena (LRA), demonstrate that this ``computation in superposition'' yields a $2\times$--$4\times$ speedup with minimal accuracy loss. This approach is particularly valuable for edge computing and embedded platforms where memory bandwidth and latency are critical bottlenecks.

\paragraph{The Binding Problem}
Standard neural networks struggle to represent and decompose joint representations of distinct entities and their attributes. This limitation is widely recognized in literature as the variable binding problem~\cite{greffBindingProblemArtificial2020, Feldman2013TheNB}. To address this problem, HDC offers a mathematically rigorous binding operation to distinctively represent object-attribute compositions.
Recent neuro-vector-symbolic architectures~\cite{hersche2023neuro} mitigate this issue by training encoder networks to map perceptual inputs (images) into vector-symbolic representations. In these architectures, objects are formed by binding attributes, and scenes are formed by bundling objects. This method has shown significant improvements in visual abstract reasoning tasks, such as Raven's Progressive Matrices (RAVEN)~\cite{johnRavenProgressiveMatrices2003}, outperforming state-of-the-art pure DNNs and neuro-symbolic baselines in both accuracy and data efficiency. Other works have similarly extended Transformers with HD representations to enhance their structural processing capabilities~\cite{bettayeb2024adapting}.

\paragraph{Hardware and System Considerations}
Beyond algorithmic improvements, the integration of HDC into NNs is often driven by hardware considerations. HDC operations are highly parallelizable and robust to noise, making them ideal candidates for emerging hardware accelerators such as FPGAs and Processing-in-Memory (PIM) architectures~\cite{salamatF5HDFastFlexible2019,karunaratneInmemoryHyperdimensionalComputing2020,geClassificationUsingHyperdimensional2020}. In these hybrid systems, feature extraction is performed by the NN, while the symbolic manipulation and reasoning are offloaded to efficient HDC operations, thereby reducing the energy cost of symbolic processing and enabling lower-precision computation without sacrificing robustness.

\subsubsection{HDC in cognitive modeling}
HDC has been employed to model a variety of cognitive capabilities such as sequence memorization and problem-solving. These models are applied to cognitive tasks including the Wason selection task \cite{eliasmithCognitionNeuronsLargescale}, n-back task \cite{gosmannSpikingNeuralModel2015}, and Raven's Progressive Matrices \cite{rasmussenNeuralModelRule2011}. For example, \cite{murdockTheoryStorageRetrieval1982} shows that using HDC as a sequence recall model can produce results that closely mimics human performance. In addition, HDC has been used as a component of cognitive architectures such as Semantic Pointer Architecture Unified Network (SPAUN) \cite{eliasmithLargeScaleModelFunctioning2012}.

\subsection{Fourier Holographic Reduced Representations}

Fourier Holographic Reduced Representations (FHRR) is a specific implementation of HDC. 
While other implementations, such as Multiply Add Permute (MAP) \cite{gayler1998multiplicative} or Holographic Reduced Representations (HRR) \cite{plate1995holographic}, rely on bipolar or real-valued vectors respectively, FHRR operates in the complex domain using phasors. FHRR naturally maintains unit magnitude during binding operations, avoiding the noise accumulation and normalization issues often prevalent in real-valued architectures. Crucially, FHRR provides a rigorous theoretical connection to kernel methods via Random Fourier Features \cite{rahimi2007random}, allowing us to interpret hyperdimensional similarity as a kernel approximation.

An FHRR hypervector is of the form $\mathbf{H}=[e^{i\theta_1},...,e^{i\theta_D}]$. Bundling ($+$) and binding ($*$) are the usual element-wise addition and multiplication, respectively. In addition, the similarity is defined as $\delta(\mathbf{H}_1,\mathbf{H}_2):=\mathrm{Re}[\mathbf{H}_1^\dagger \mathbf{H}_2]/D$. Here, $\dagger$ denotes the conjugate transpose and $\mathrm{Re}(z)$ is the real component of $z\in\mathbb{C}$.


Fractional power encoding (FPE) has been used for encoding data points to hypervectors. It is a locality-preserving encoding for points on a data manifold where the similarity, or inner product, of the hypervectors reflects the relationship between the points~\cite{plate1994distributed, fradyComputingFunctionsUsing2022}. For the $k$-th feature of data of dimension $n$, attached is an FHRR base hypervector of the form $\mathbf{H}_k = e^{i\theta}$, where $\theta$ is a column vector such that $\theta_j \sim p_k$ for some fixed distribution $p_k$. The data point is then encoded as the binding of the base hypervector exponentiated by the value of the corresponding feature: $\phi(\mathbf{x})=\mathbf{H}_1^{\mathbf{x}_1}*\mathbf{H}_2^{\mathbf{x}_2}*...*\mathbf{H}_n^{\mathbf{x}_n}$. Here, the exponentiation $\mathbf{H}_j^{\mathbf{x}_j}$ refers to exponentiating each element of $\mathbf{H}_j$ by $\mathbf{x}_j$. This allows the data to be smoothly expressed and manipulated, as representations generated this way are correlated according to some shift-invariant kernel on the corresponding feature values, which we will explain below.

There is a remarkable connection between FPE of FHRR hypervectors and kernel-based methods, making it a powerful tool for learning in HDC. In particular, it coincides with the Random Fourier Features (RFF) \cite{rahimiRandomFeaturesLargeScale} encoding, an efficient approximation of kernel methods. The RFF encoding is a map $\phi:\mathbb{R}^n\to\mathbb{C}^D$, with $\phi(\mathbf{x})=e^{i\mathbf{Mx}}$, where each row $\mathbf{M}_{j,:}\sim p$ for some multivariate distribution $p$. The columns $\mathbf{M}_{:,k}$ can then be viewed as the exponents of the base hypervector $\mathbf{H}_k$, and $p = [p_1, ..., p_n]$. As a result of Bochner's theorem, $\langle\phi(\mathbf{x}),\phi(\mathbf{y})\rangle/D\approx K(\mathbf{x}-\mathbf{y})$, where $K$ is a shift-invariant kernel that is the Fourier transform of distribution $p$ \cite{rahimiRandomFeaturesLargeScale} (See Supplementary Materials A). Notably, when $p$ is the standard Gaussian distribution, the radial basis function (RBF) kernel is recovered. Many HDC learning models achieve good time and accuracy performance by the efficient kernel approximation coupled with some HDC algorithm.

\section{Overview of GHRR}\label{sec:ghrr-overview}

We introduce Generalized Holographic Reduced Representations (GHRR), a framework extending Fourier Holographic Reduced Representations (FHRR). While standard FHRR operates on scalar components in the unitary group $U(1)$, GHRR generalizes base elements to unitary matrices in $U(m)$ for $m \ge 1$.

This generalization addresses two fundamental limitations where scalar-based HDC methods fail:
\begin{itemize}
    \item \textbf{Non-Commutativity:} Standard binding (element-wise scalar multiplication) is commutative, failing to encode asymmetric structures like sequences or trees without permutations. GHRR binding is inherently non-commutative, enabling the natural representation of ordered structures.
    \item \textbf{Expressive Binding:} Standard binding acts as a low-capacity holographic projection. It fails to retain sufficient information for complex interactions such as attention. GHRR binding (block-wise matrix multiplication) offers higher expressivity, interpolating between parsimonious projection and full tensor products.
\end{itemize} 

GHRR hypervectors are tensors of the form
\begin{align}
    \mathbf{H}=[\mathbf{A}^{(1)},...,\mathbf{A}^{(D)}]^\top\in\mathbb{C}^{D\times m\times m},
\end{align}
where $\mathbf{A}^{(j)}\in\mathrm{U}(m)$ such that $\mathbf{A}^{(j)}\sim p$ for some distribution $p$ over $\mathrm{U}(m)$, for $j=1,...,D$.

The standard operations in HDC extend naturally. Let $\mathbf{H}_1=[\mathbf{A}^{(j)}]_{j=1}^D$ and $\mathbf{H}_2=[\mathbf{B}^{(j)}]_{j=1}^D\in\mathbb{C}^{D\times m\times m}$. We define bundling as element-wise addition, i.e.
\begin{align}
    \mathbf{H}_1+\mathbf{H}_2 = [\mathbf{A}^{(j)}+\mathbf{B}^{(j)}]_{j=1}^D,
\end{align}
and binding to be element-wise matrix multiplication, i.e.
\begin{align}
    \mathbf{H}_1*\mathbf{H}_2 = [\mathbf{A}^{(j)}\mathbf{B}^{(j)}]_{j=1}^D.
\end{align}
where $\mathbf{A}^{(j)}\mathbf{B}^{(j)}$ is the matrix-product between $\mathbf{A}^{(j)}$ and $\mathbf{B}^{(j)}$. \textit{For notational consistency, $*$ will always refer to element-wise matrix multiplication unless specified otherwise.} When binding multiple hypervectors, we use the notation $\mathbf{H}_1*\dots * \mathbf{H}_n=\prod_{k=1}^n\mathbf{H}_k$.

We define the similarity between two hypervectors as
\begin{align}\label{eq:sim}
    \delta(\mathbf{H}_1,\mathbf{H}_2)=\frac{1}{mD}\mathrm{Re}\,\text{tr}\left(\sum_{j=1}^D \mathbf{A}^{(j)}(\mathbf{B}^{(j)})^\dagger\right).
\end{align}
It can be seen that for $m=1$, this similarity is exactly that for FHRR. Note that after bundling, the matrix entries need not be unitary. However, we still require that base hypervectors be unitary such that (1) they are similar to themselves under the similarity defined in Eq.~\ref{eq:sim}, and (2) the base hypervectors and their fractional power are norm-preserving with respect to binding~\cite{fradyComputingFunctionsUsing2022}. Figure \ref{fig:ghrr-overview} provides an overview of and comparison between GHRR and FHRR.

\begin{figure}
    \centering
    \includegraphics[width=\columnwidth]{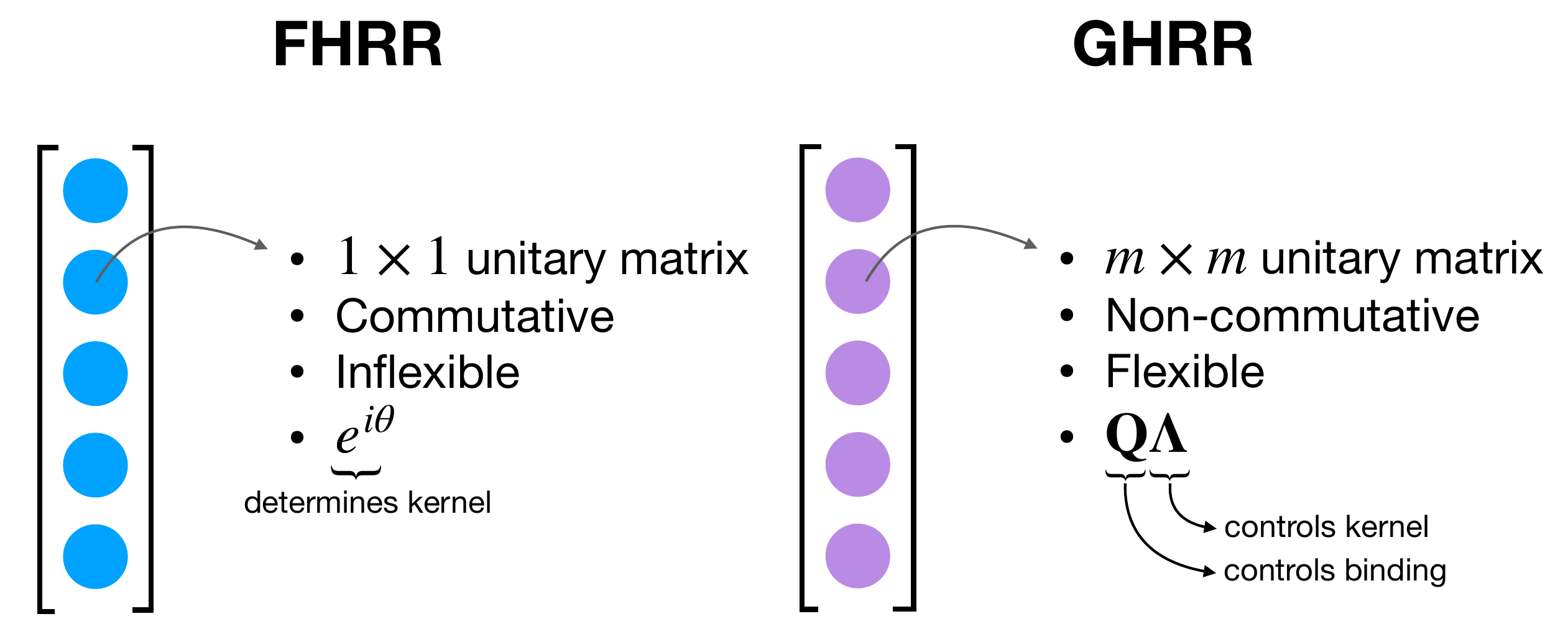}
    \caption{Comparison between FHRR and GHRR.}
    \label{fig:ghrr-overview}
\end{figure}

\section{An Implementation of GHRR}\label{sec:ghrr-implementation}
What we have described in Section \ref{sec:ghrr-overview} are general characteristics of GHRR. Specifying an implementation entails specifying (1) the form of the unitary matrix components; and, relatedly, (2) how they are sampled. 

\subsection{Desired Properties}\label{subsec:ghrr-prop}

We want our GHRR base hypervectors to fulfill the basic properties of HD representations. Most notably, we need the binding operation for representing the association of features works as intended~\cite{gayler2004vector}. For independently sampled  base hypervectors $\mathbf{H}_1,\mathbf{H}_2,\mathbf{H}_3$,
\begin{enumerate}
    \item $\delta(\mathbf{H}_1,\mathbf{H}_2)\to 0$ as $D\to\infty$; and
    \item $\delta(\mathbf{H}_1*\mathbf{H}_2,\mathbf{H}_1)\to 0$ as $D\to\infty$.
    \item $\delta(\mathbf{H}_1,\mathbf{H}_2)\approx\delta(\mathbf{H}_3*\mathbf{H}_1,\mathbf{H}_3*\mathbf{H}_2)\approx\delta(\mathbf{H}_1*\mathbf{H}_3,\mathbf{H}_2*\mathbf{H}_3)$.
\end{enumerate}
While property 3 follows naturally from the cyclic property of the trace, the other two properties do not hold in arbitrary implementations. In addition, since GHRR is an extension of FHRR, we want the two to be equivalent when $m=1$.

Property 1 specifies that hypervectors should be quasi-orthogonal. This is motivated by the fact that base hypervectors are typically used to represent discrete abstract symbols. As such symbols are purely formal and lack semantic content, their corresponding representations should be uncorrelated. Property 2 ensures that the result of the binding operation is quasi-orthogonal to its operands. This is crucial for distinguishing a composite structure from its constituent parts; it guarantees that the representation of an association (e.g., a key-value pair) is distinct from the representations of the individual items, thereby preventing ambiguity and interference during the retrieval process. Finally, Property 3 ensures that the similarity between two concepts is invariant under binding, guaranteeing that relational structures are maintained even when concepts are embedded into more complex representations.

\subsection{Quasi-orthogonal Representations}\label{subsec:ghrr-quasi}
We derive sufficient conditions for these properties to hold, and based on that, develop a constrained parameterization for GHRR representations.

\begin{proposition}\label{prop:orthogonal}
Let $\mathbf{A}=\mathbf{Q}_{1}\mathbf{\Lambda}_{1}$ and $\mathbf{B}=\mathbf{Q}_{2}\mathbf{\Lambda}_{2} \in \mathbb{C}^{m\times m}$ be random matrices where $\mathbf{Q}_{j}$ and $\mathbf{\Lambda}_{j}$ are sampled independently of each other for $j=1,2.$ Moreover, let $\mathbf{\Lambda}_{1} = \operatorname{diag}(\lambda_{1},...,\lambda_{m})$ and $\mathbf{\Lambda}_{2} = \operatorname{diag}(\eta_{1},...,\eta_{m})$ with $\lambda_{j},\eta_{j} \in U(1)$ for $j=1,...,m$. If $\mathbb{E}[\eta_{j}\lambda_{j}^{*}]=0$ for all $j=1,...,m,$ then $\mathbb{E}[\operatorname{tr}(\mathbf{A}\mathbf{B}^{\dagger})]=0.$
\end{proposition}

\begin{proof}
To establish quasi-orthogonality (Property 1), we must show that the expected similarity between two independently generated hypervectors is zero. This similarity is defined by the trace of their product $\mathbf{A}\mathbf{B}^\dagger$.

First, we substitute the definitions of $\mathbf{A}$ and $\mathbf{B}$:
\begin{equation}
\operatorname{tr}(\mathbf{A}\mathbf{B}^{\dagger}) = \operatorname{tr}(\mathbf{Q}_{1}\mathbf{\Lambda}_{1}\mathbf{\Lambda}_{2}^{\dagger}\mathbf{Q}_{2}^{\dagger})
\end{equation}

By utilizing the cyclic property of the trace ($\operatorname{tr}(\mathbf{X}\mathbf{Y}\mathbf{Z}) = \operatorname{tr}(\mathbf{Z}\mathbf{X}\mathbf{Y})$), we can rearrange the terms to group the unitary matrices $\mathbf{Q}$ together:
\begin{equation}
\operatorname{tr}(\mathbf{A}\mathbf{B}^{\dagger}) = \operatorname{tr}((\mathbf{Q}_{2}^{\dagger}\mathbf{Q}_{1})\mathbf{\Lambda}_{1}\mathbf{\Lambda}_{2}^{\dagger})
\end{equation}

Let $\mathbf{Q} = \mathbf{Q}_{2}^{\dagger}\mathbf{Q}_{1}$. The matrix $\mathbf{Q}$ represents the relative rotation between the two hypervectors' bases. Since $\mathbf{\Lambda}_1$ and $\mathbf{\Lambda}_2^\dagger$ are diagonal matrices, their product forms a diagonal matrix with entries $\lambda_k \eta_k^*$. The trace of the product of a general matrix $\mathbf{Q}$ and a diagonal matrix involves only the diagonal elements of $\mathbf{Q}$. Thus, the expression simplifies to a weighted sum:
\begin{equation}
\operatorname{tr}(\mathbf{A}\mathbf{B}^{\dagger}) = \sum_{k=1}^{m} Q_{kk} \lambda_{k}\eta_{k}^{*}
\end{equation}

Finally, we take the expectation of this sum. Because the phase factors $\lambda_k$ and $\eta_k$ are sampled independently from the rotation matrix $\mathbf{Q}$ and satisfy the zero-mean condition $\mathbb{E}[\lambda_{k}\eta_{k}^{*}] = 0$, the expectation distributes over the sum and vanishes:
\begin{align}
\mathbb{E}[\operatorname{tr}(\mathbf{A}\mathbf{B}^{\dagger})] &= \sum_{k=1}^{m} \mathbb{E}[Q_{kk}] \mathbb{E}[\lambda_{k}\eta_{k}^{*}]\label{eq:tr-exp-val}
= \sum_{k=1}^{m} \mathbb{E}[Q_{kk}] \cdot 0 = 0
\end{align}
\end{proof}

\begin{corollary}\label{cor:orthogonal}
    Suppose we sample random matrices $\mathbf{A},\mathbf{B}$ with the form $\mathbf{Q\boldsymbol{\Lambda}}$, where $\mathbf{Q}\in\mathbb{C}^{m\times m}$ is a randomly sampled unitary matrix and, independently, $\mathbf{\boldsymbol{\Lambda}}=\mathrm{diag}(\boldsymbol{\Lambda}_1,...,\boldsymbol{\Lambda}_m)$, where $\boldsymbol{\Lambda}_j\sim p_j$ where $p_j$ is a symmetric distribution with zero mean for $j=1,...,m$. Then $\mathbb{E}\,\mathrm{tr}(\mathbf{AB}^\dagger)=0$.
\end{corollary}

\begin{corollary}\label{cor:binding}
    Suppose we sample random matrices $\mathbf{A},\mathbf{B}$ according to Corollary \ref{cor:orthogonal}. Then $\mathbb{E}\,\mathrm{tr}(\mathbf{A}(\mathbf{AB})^\dagger)=0$.
\end{corollary}
\begin{proof}
    Observe that
    \begin{align}
        \mathbb{E}\,\mathrm{tr}(\mathbf{A}(\mathbf{AB})^\dagger) &= \mathbb{E}\,\mathrm{tr}(\mathbf{A}\mathbf{B}^\dagger\mathbf{A}^\dagger)
        = \mathbb{E}\,\mathrm{tr}(\mathbf{I}\mathbf{B}^\dagger).
    \end{align}
    Since $\mathbf{B}=\mathbf{Q\boldsymbol{\Lambda}}$ where $\mathbf{\boldsymbol{\Lambda}}=\mathrm{diag}(\boldsymbol{\Lambda}_1,...,\boldsymbol{\Lambda}_m)$, with $\boldsymbol{\Lambda}_j$ sampled from distributions with zero mean, $\mathbb{E}\,\boldsymbol{\Lambda}_j=0$ for all $j=1,...,m$. So Proposition \ref{prop:orthogonal} applies and we are done.
\end{proof}

{

Quasi-orthogonality is also observed empirically. See Section \ref{sec:enc-data} and Figure \ref{fig:fixed-vs-vary} for details. 
}

Corollary \ref{cor:orthogonal} gives us a way to construct base hypervectors satisfying the desiderata mentioned above. By default, Corollary \ref{cor:orthogonal} gives us desired property 1, while Corollary \ref{cor:binding} gives us desired property 2. Moreover, it is evident that when $m=1$, this GHRR implementation is equivalent to FHRR. Thus, a GHRR base hypervector is of the form
\begin{align}[\mathbf{Q}^{(1)}\mathbf{\boldsymbol{\Lambda}}^{(1)},...,\mathbf{Q}^{(D)}\mathbf{\boldsymbol{\Lambda}}^{(D)}]^\top,
\end{align}
where $\mathbf{Q}^{(j)}\in \mathrm{U}(m)$ and $\mathbf{\boldsymbol{\Lambda}}^{(j)}=\mathrm{diag}(e^{i\theta_1},...,e^{i\theta_m})$, with $\theta_k\sim p_k$ such that $\mathbb{E}[e^{i\theta_k}]=0$. Figure \ref{fig:quasi-orthogonal} shows a histogram of the similarity between randomly sampled hypervectors following the scheme described in Corollary \ref{cor:orthogonal}, where $D=1000$, $m=3$, and $p_j=\mathrm{Unif}(0,2\pi)$. We observe that the randomly sampled hypervectors are quasi-orthogonal; i.e. their similarities are concentrated about zero. The top histogram holds $\mathbf{Q}^{(j)}=\mathbf{Q}^{(k)}$ for all $j,k=1,...,D$, while the middle histogram has randomly sampled $\mathbf{Q}^{(j)}$. The bottom histogram shows that the similarity between $H_1$ and $H_1*H_2$ respects the result in Corollary \ref{cor:binding}.

\begin{figure}
    \centering
    \includegraphics[width=\columnwidth]{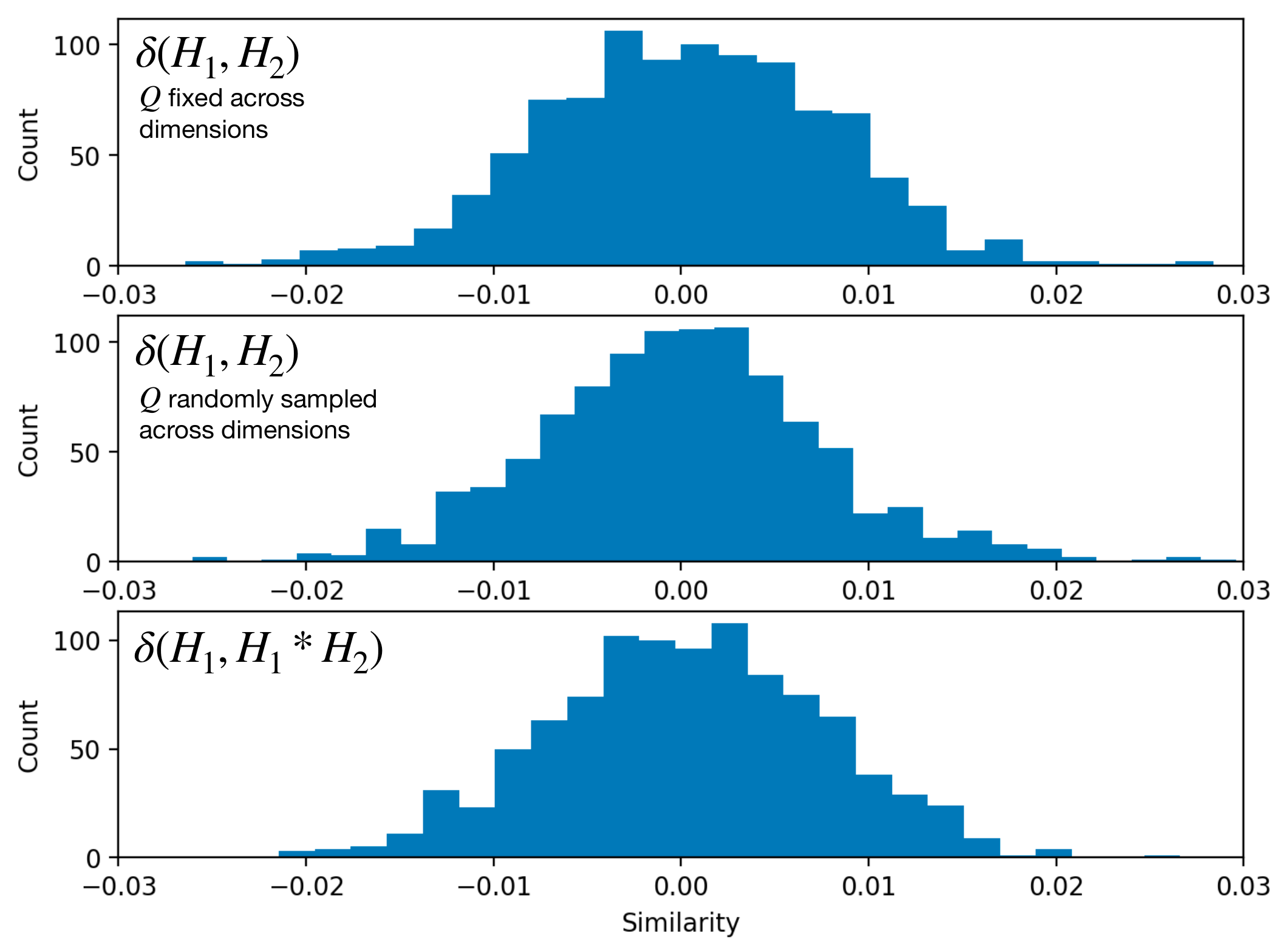}
    \caption{A histogram of the similarity between randomly sampled hypervectors following the scheme described in Corollary \ref{cor:orthogonal}, where $D=1000$, $m=3$, and $p_j=\mathrm{Unif}(0,2\pi)$. Top: Histogram where $\mathbf{Q}_j=\mathbf{Q}_k$ holds for all $j,k=1,...,D$. Middle: histogram where $\mathbf{Q}_j$ for $j=1,...,D$ are randomly sampled. Bottom: histogram of similarity between $H_1$ and $H_1*H_2$.}
    \label{fig:quasi-orthogonal}
\end{figure}

\subsection{Encoding Data}\label{sec:enc-data}

As mentioned in the introduction, RFF enables the efficient approximation of shift-invariant kernels via taking the similarity between FHRR hypervectors. To be precise, Fractional Power Encoding (FPE)~\cite{fradyComputingFunctionsUsing2022, plate1992holographic, plate1994distributed} with binding-based composition gives rise to the RFF encoding.

FPE encodes the scalar value $x \in \mathbb{R}$ of a variable $\alpha$ by exponentiating a base FHRR hypervector $\mathbf{H}_\alpha\in\mathbb{C}^{D}$ representing the variable $x$:
\begin{align}
    \phi_\alpha(x)=\mathbf{H}_\alpha^x=\left[e^{i \theta^{(j)} x}\right]_{j=1}^D, \theta^{(j)} \sim p_\alpha,
\end{align}
where $p_\alpha$ is some probability distribution. We can represent multidimensional data through binding-based composition. Namely, for an $n$-dimensional vector $\mathbf{x} \in \mathbb{R}^n$, we encode each component via FPE as above and bind all of the results together: 
\begin{align}
    \phi(\mathbf{x})=\prod_{k=1}^n \mathbf{H}_k^{\mathbf{x}_k}=
    \prod_{k=1}^n \left[e^{i \theta^{(j)}_k \mathbf{x}_k}\right]_{j=1}^D = 
    \left[e^{i (\theta^{(j)})^{\top} \mathbf{x}}\right]_{j=1}^D,
\end{align}
where $\theta^{(j)}_k \sim p_k$ and $\theta^{(j)}=[\theta^{(j)}_1,...,\theta^{(j)}_n]$. Here, the product $\prod$ denotes the binding operation for FHRR, i.e., element-wise multiplication.

In this section, we provide a viable extension of the encoding for GHRR. Let $\mathcal{H}$ be the space of GHRR representations. Suppose we have data from some input space $\mathcal{X}$. We are interested in a smooth map $\phi:\mathcal{X}\to\mathcal{H}$. As per the previous section, we factor $\phi$ into two components and write $\phi(\mathbf{x})_j=\mathbf{Q}^{(j)}(\mathbf{x})\mathbf{\boldsymbol{\Lambda}}^{(j)}(\mathbf{x})$ for $j=1,...,D$, where $\mathbf{Q}^{(j)}:\mathcal{X}\to \mathrm{U}(m)$ and $\mathbf{\boldsymbol{\Lambda}}^{(j)}:\mathcal{X}\to \{\mathrm{diag}(\boldsymbol{\Lambda}_{j1},...,\boldsymbol{\Lambda}_{jm})\,|\,\boldsymbol{\Lambda}_{jk}\in\mathbb{C},|\boldsymbol{\Lambda}_{jk}|=1\}$. In particular, for $\mathbf{\boldsymbol{\Lambda}}^{(j)}$, we choose diagonal entries of the form $\boldsymbol{\Lambda}_j(\mathbf{x})=e^{i\mathbf{w}_{jk}^\top \mathbf{x}}$, where $\mathbf{w}_{jk}\sim p_k$ with $p_k$ being symmetric distributions with zero mean. We can interpret this component of $\phi$ as the part that controls the basic shape of the kernel as in FHRR.

For $\mathbf{Q}$, there are two orthogonal choices to make: whether or not $\mathbf{Q}$ should depend on the input $\mathbf{x}$; and whether or not there should be variance over $j$. These two options give us four classes of GHRR encodings. We consider the case where $\mathbf{Q}$ does not depend on $\mathbf{x}$. Thus, we have the encoding
\begin{align}
    \phi(\mathbf{x})_j=\mathbf{Q}^{(j)}\mathbf{\boldsymbol{\Lambda}}^{(j)}(\mathbf{x}).
\end{align}

\subsubsection{Kernel Properties of GHRR}

Suppose we have two encodings $\phi_1,\phi_2$, where $\phi_1^{(j)}(\mathbf{x})=\mathbf{Q}_{1}^{(j)}\mathbf{\boldsymbol{\Lambda}}^{(j)}(x)$ and $\phi_2^{(j)}(\mathbf{x})_j=\mathbf{Q}_2^{(j)}\mathbf{\boldsymbol{\Lambda}}^{(j)}(x)$ where $\mathbf{Q}_{1}^{(j)}\sim q_1$ and $\mathbf{Q}_{2}^{(j)}\sim q_2$ for $j=1,...,D$. Let $\mathbf{\boldsymbol{\Lambda}}^{(j)}(\mathbf{x})=\mathrm{diag}(e^{i\mathbf{w}_{j1}^\top\mathbf{x}},...,e^{i\mathbf{w}_{jm}^\top\mathbf{x}})$, where $\mathbf{w}_{jk}\sim p_k$. Moreover, let $\mathbf{Q}^{(j)}=(\mathbf{Q}_2^{(j)})^\dagger\mathbf{Q}_1^{(j)}\sim q$ where $q$ is a distribution induced by $q_1$ and $q_2$. Then we have
\begin{align}
    \delta(\phi_1(\mathbf{x}),\phi_2(\mathbf{y})) &\approx\frac{1}{m}\mathbb{E}\,\mathrm{tr}(\mathbf{Q}^{(j)}\mathbf{\boldsymbol{\Lambda}}^{(j)}(\mathbf{x}) \mathbf{\boldsymbol{\Lambda}}^{(j)}(\mathbf{y})^\dagger)\\
    &= \frac{1}{m}\mathrm{Re}\left(\sum_{k=1}^m \mathbb{E}_{q}[\mathbf{Q}^{(j)}_{kk}]\mathbb{E}_{p_k}[e^{i\mathbf{w}_{jk}^\top (\mathbf{x}-\mathbf{y})}]\right)
\end{align}
where $\mathbf{w}_{jk}\sim p_k$ for $k=1,...,m$. The last step assumes that $\mathbf{Q}^{(j)}_{kk}$ and $\mathbf{w}_{jk}$ are generated independently. This allows us to seperate the expression into a product of two parts: one that depends on the data in a shift-invariant manner, and one that is independent of the data.

We would like to express this expression as a sum of kernels. We can write $\mathbf{Q}^{(j)}_{kk}=r_ke^{i\theta_k}$, where $(r_1,...,r_m,\theta_1,...,\theta_m)\sim g$, with $g$ being some distribution induced by $q$. In addition, let $\Theta_k$ be $g$ marginalized over all variables except for $\theta_k$ and $R_k:=g(\cdot|\theta_k)$ be a conditional distribution marginalized over all variables except for $r_k$ and $\theta_k$. Then
\begin{align}
    \delta(\phi_1(\mathbf{x}),\phi_2(\mathbf{y}))
    &\approx \frac{1}{m}\sum_{k=1}^m \mathbb{E}_{R_k}[r_k]\mathrm{Re}(\mathbb{E}_{p_k,\Theta_k}[e^{i\mathbf{w}_{jk}^\top (\mathbf{x}-\mathbf{y})+i\theta_k}])\\
    &=\frac{1}{m}\sum_{k=1}^m \mathbb{E}_{R_k}[r_k]\mathrm{Re}(\mathbb{E}_{\Tilde{p}_k}[e^{i\Tilde{\mathbf{w}}_{jk}^\top (\Tilde{\mathbf{x}}-\Tilde{\mathbf{y}})}])\\
    &=\frac{1}{m}\sum_{k=1}^m \mathbb{E}_{R}[r_k]\Tilde{K}_k(\Tilde{\mathbf{x}}-\Tilde{\mathbf{y}})\label{eq:weighted-kernel}
\end{align}
Here, $\Tilde{\mathbf{w}}_{jk}=[\mathbf{w}_{jk},\theta_k]\sim \Tilde{p}_k=[p_k,\Theta_k]$, $\Tilde{\mathbf{x}}=[\mathbf{x},1]$, $\Tilde{\mathbf{y}}=[\mathbf{y},0]$, and $\Tilde{K}_k$ is the kernel corresponding to $\Tilde{p}_k$.
Thus the kernel corresponding to the similarity between encodings with different $\mathbf{Q}$ but the same $\mathbf{\boldsymbol{\Lambda}}(\mathbf{x})$ is a weighted sum of the kernels, where each kernel is influenced by both the diagonal elements of $\mathbf{\boldsymbol{\Lambda}}(\mathbf{x})$ as well as the distribution over $\mathbf{Q}$. The weights are solely determined the distribution over $\mathbf{Q}$. As a special case, consider $\phi_1=\phi_2$. Then $\mathbf{Q}_{kk}=1$, resulting in the kernel $\frac{1}{m}\sum_{k=1}^m K_k(\mathbf{x}-\mathbf{y})$, where $K_k$ is the kernel corresponding to $p_k$. This realization gives us a way to interpret an encoding $\phi$ where $\mathbf{Q}$ depends on $\mathbf{x}$ as an adaptive kernel on the input. 

\begin{table*}[t]

    \centering
    \caption{Summary of Main GHRR Variants and Qualitative Effects}
    \label{tab:ghrr_variants}
    \renewcommand{\arraystretch}{1.5}
    \begin{tabular}{p{0.25\linewidth} p{0.3\linewidth} p{0.35\linewidth}}
        \hline
        \textbf{Variant} & \textbf{Structure of $\mathbf{Q}$} & \textbf{Key Properties} \\
        \hline
        \textbf{Fixed $\mathbf{Q}$} & Fixed across dimensions; Independent of input & Simplest implementation retaining non-commutative binding; exhibits quasi-orthogonality with higher variance. \\
        \hline
        \textbf{Varying $\mathbf{Q}$} & Varies across dimensions (sampled i.i.d.); Independent of input & Exhibits quasi-orthogonality with lower variance.\\
        \hline
        \textbf{Input-dependent $\mathbf{Q}$} & Dependent on input ($\mathbf{Q}(x)$); Fixed or Varying across dimensions & Enables adaptive kernel shapes where feature importance is modulated by context; functions as a re-weighting of kernels, allowing for attention-like mechanisms and higher expressivity. \\
        \hline
    \end{tabular}
    
\end{table*}

Figure \ref{fig:fixed-vs-vary} demonstrates the difference between having a fixed versus a varying $\mathbf{Q}$ across dimensions, where the top and bottom histogram visualizes the distribution of $\delta(\phi_1(0),\phi_2(0))$ where $\mathbf{Q}$ is fixed across dimensions and varied across dimensions, respectively. Specifically, let $\mathbf{Q}=[\mathbf{Q}_{j}]_{j=1}^D\in\mathbb{C}^{D\times m\times m}$. If $\mathbf{Q}$ is fixed, $\mathbf{Q}_j$ is a randomly sampled unitary matrix and we set $\mathbf{Q}_i=\mathbf{Q}_j$ for all $i,j$. When $\mathbf{Q}$ is varying, each $\mathbf{Q}_j$ is sampled i.i.d. for all $j$. 

We observe that by varying $\mathbf{Q}$, the distribution is significantly more centered about the mean. When the mean is close to zero, this suggests that using randomly sampled encodings with varying $\mathbf{Q}$ can minimize cross-talk interference by default. On the other hand, in an encoding scheme where $\mathbf{Q}$ is learnable, one can more easily optimize $\mathbf{Q}$ to exhibit desired behavior, as there are fewer parameters to optimize.

\begin{figure}
    \centering
    \includegraphics[width=\columnwidth]{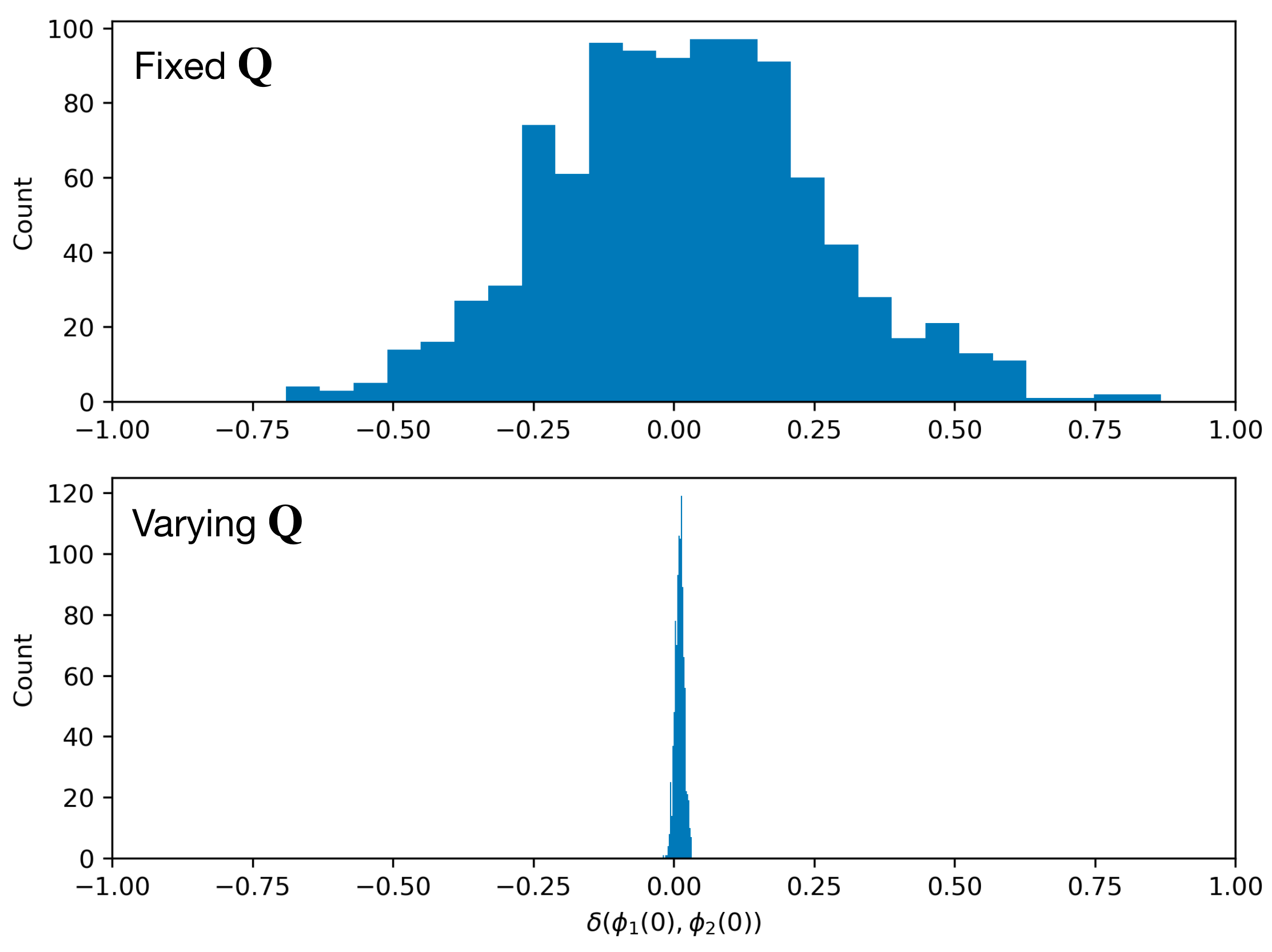}
    \caption{Distribution of $\delta(\phi_1(0),\phi_2(0))$ over encodings $\phi_1,\phi_2$. Top: $\mathbf{Q}$ is fixed across dimensions for $\phi_1,\phi_2$. Bottom: $\mathbf{Q}$ is varied across dimensions for $\phi_1,\phi_2$.}
    \label{fig:fixed-vs-vary}
\end{figure}

\subsection{Binding as holographic projections}\label{subsec:binding}

While kernel properties between two encodings can be described simply, the power of GHRR comes from its ability to perform binding while retaining more information regarding its constituents. 
Binding two hypervectors in FHRR can be viewed as taking the tensor product of the two hypervectors and then taking the diagonal; i.e. a holographic projection. In the case of GHRR, we can view binding in GHRR as an extension of binding in FHRR. In particular, let the effective dimension of a GHRR encoding be $Dm$ and suppose it is fixed. By modulating $m$, we can interpolate between binding as a parsimonious holographic projection as in the case of FHRR and binding as an equivalent of taking the full tensor product. Figure \ref{fig:ghrr-binding} illustrates the intuition behind binding in FHRR and GHRR as described above. The fully colored square represents the full tensor product representation resulting from binding two hypervectors together; binding in FHRR/GHRR ``collapses'' this representation by taking the diagonal/block-diagonal respectively.

To illustrate this point, consider two hypervectors $\mathbf{G}_1=[\mathbf{Q}^{(j)}\mathbf{\boldsymbol{\Lambda}}^{(j)}]_{j=1}^D$ and $\mathbf{G}_2=[\mathbf{R}^{(j)}\mathbf{H}^{(j)}]_{j=1}^D$. Let us focus on the $j$-th matrix-element and let $\mathbf{\boldsymbol{\Lambda}}^{(j)}=\mathrm{diag}(\boldsymbol{\Lambda}_{j1},...,\boldsymbol{\Lambda}_{jm})$ and $\mathbf{H}^{(j)}=\mathrm{diag}(\eta_{j1},...,\eta_{jm})$. Moreover, denote $\mathbf{Q}^{(j)}_{kl}=q_{kl}$ and $\mathbf{R}^{(j)}_{kl}=r_{kl}$. Then the corresponding element of $\mathbf{G}_1*\mathbf{G}_2$ is
\begin{align}\label{eq:elem-bind}
    [\mathbf{G}_1*\mathbf{G}_2]^{(j)} &= \mathbf{Q}^{(j)}\mathbf{\boldsymbol{\Lambda}}^{(j)}\mathbf{R}^{(j)}\mathbf{H}^{(j)}
    = \left[\sum_{n=1}^mq_{kn}r_{nl}\boldsymbol{\Lambda}_{jn}\eta_{jl}\right]_{k,l=1}^m
\end{align}
Thus, each entry of the resulting product of matrix-elements is a linear combination of the entries of the tensor product $\boldsymbol{\Lambda}_j\eta_j^\top$, where $\boldsymbol{\Lambda}_j=[\boldsymbol{\Lambda}_{j1},...,\boldsymbol{\Lambda}_{jm}]$ and $\eta_j=[\eta_{j1},...,\eta_{jm}]$, suggesting that it is some transformed ``view'' of the tensor product determined by $\mathbf{Q}^{(j)}$ and $\mathbf{R}^{(j)}$. Specifically, let $\mathbf{U}_{jl}=\mathbf{Q}^{(j)}\mathrm{diag}(\mathbf{R}^{(j)}_{:,l})$, where $\mathbf{R}^{(j)}_{:,l}$ is the $l$-th column of $\mathbf{R}^{(j)}$. Then Eq.~\ref{eq:elem-bind} defines a map
\begin{align}\label{eq:transform}
    \varphi_j:\boldsymbol{\Lambda}_j\eta_j^\top\mapsto \left[\mathbf{U}_{jl}\mathbf{v}_{jl}\right]_{l=1}^m,
\end{align}
where $\mathbf{v}_{jl}=\boldsymbol{\Lambda}_j\eta_{jl}$. Clearly, provided $\mathbf{R}^{(j)}$ has all non-zero entries, $\varphi_j$ is invertible.

Let $\boldsymbol{\Lambda}=[\boldsymbol{\Lambda}_1,...,\boldsymbol{\Lambda}_D]$ and $\eta=[\eta_1,...,\eta_D]$ be the concatenations of $\boldsymbol{\Lambda}_1,...,\boldsymbol{\Lambda}_D$ and $\eta_1,...,\eta_D$, respectively. Now, taking the entire $\mathbf{H}_1*\mathbf{H}_2$ into account, we can interpret it as a holographic projection of the tensor product $\boldsymbol{\Lambda}\eta^\top$, but instead of just taking the diagonal, we take the block-diagonal where blocks are of size $m$ and subsequently transformed by $\mathbf{Q}^{(j)}$ and $\mathbf{R}^{(j)}$ as defined by $\varphi_j$.

For fixed effective dimension, when $m$ is minimal, i.e. 1, then the block-diagonal is just the diagonal as in FHRR, while, when $m$ is maximal, the block-diagonal is the whole tensor product, i.e. as in Tensor Product Representations \cite{smolenskyTensorProductVariable1990}. 

\begin{figure}
    \centering
    \includegraphics[width=\columnwidth]{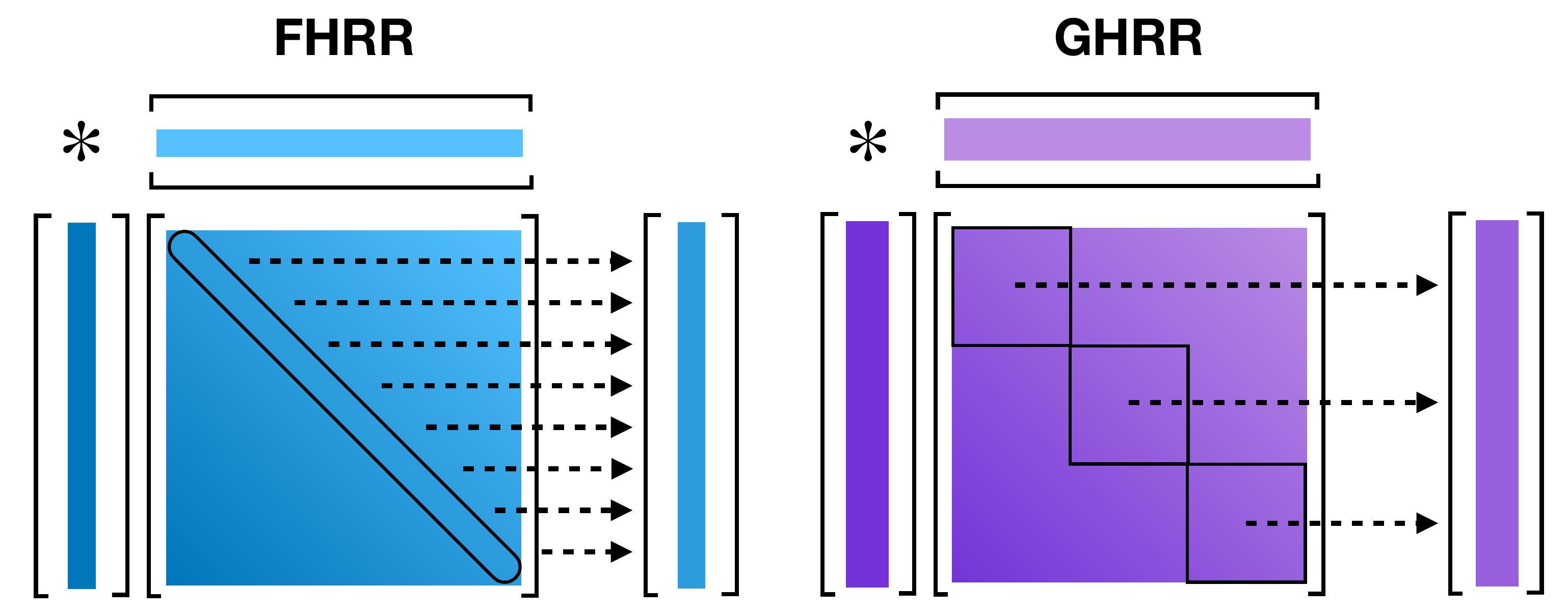}
    \caption{A visualization of binding in FHRR and GHRR. Left: FHRR binding as a projection of the diagonal of the outer product matrix. Right: GHRR binding as a projection of the block-diagonal of the outer product matrix.}
    \label{fig:ghrr-binding}
\end{figure}

From a neural perspective, we can interpret the tensor product between two vectors as representing all possible pairwise connections between the dimensions; i.e. they are fully connected. In contrast, the diagonal projection of a tensor product represents sparse connections where only the corresponding dimensions of the vectors are connected. A block-diagonal projection takes the middle ground, representing connections between nearby dimensions as demarcated by the blocks.

\subsection{Binding implements attention} \label{subsec:binding-attention}

As shown in Section \ref{subsec:binding}, binding in GHRR is the block-diagonal generalization of binding in FHRR. What does this greater level of expressivity afford us? In other HDC frameworks, binding enables the conjunction of representations. As binding in GHRR is more expressive due to matrix multiplication, it can represent more complex operations. In this section, we show that binding in GHRR can implement attention. 

We focus on scaled dot-product attention \cite{vaswaniAttentionAllYou2017}, which can be written as $\mathrm{attn}(\mathbf{Q}, \mathbf{K}, \mathbf{V})=\mathrm{softmax}\left(\frac{1}{\sqrt{d_k}}\mathbf{Q}\mathbf{K}^\top\right)\mathbf{V}$, where $\mathbf{Q}, \mathbf{K}, \mathbf{V}$ are the embeddings of the input features corresponding to query, key, and value respectively, and ${1}/{\sqrt{d_k}}$ is the scaling factor determined by the embedding dimension $d_k$, which we omit for simplicity. Generally, in the standard Transformer architecture, $\mathbf{Q}$, $\mathbf{K}$, and $\mathbf{V}$ are defined in terms of input sequences of token embeddings $\mathbf{X}_q, \mathbf{X}_k, \mathbf{X}_v \in \mathbb{R}^{N \times d_{\text{model}}}$. The terms $\mathbf{W}_q, \mathbf{W}_k, \mathbf{W}_v \in \mathbb{R}^{d_{\text{model}} \times d_k}$ represent learnable linear projection matrices that map the input embeddings into the query, key, and value spaces, respectively. Then, we have:
\begin{align}\label{eq:attn2}
    \text{attn}(\mathbf{Q}, \mathbf{K}, \mathbf{V}) = \mathrm{softmax}(\mathbf{X}_q \mathbf{W}_q \mathbf{W}_k^\top \mathbf{X}_k^\top) \mathbf{X}_v \mathbf{W}_v 
\end{align}

We first show that it is possible to match the mathematical form of the attention mechanism given in Eq.~\ref{eq:attn2} using GHRR. We then show that, by binding specific positional information to GHRR token hypervectors, we can implement token-level attention using GHRR.

\paragraph{Matching mathematical forms between GHRR and attention} To demonstrate the equivalence in GHRR, suppose we have three GHRR hypervectors decomposed as $\mathbf{Q}=\mathbf{W}_q\boldsymbol{\Lambda}_q$, $\mathbf{K}=\mathbf{W}_k\boldsymbol{\Lambda}_k$, and $\mathbf{V}=\mathbf{W}_v\boldsymbol{\Lambda}_v$. Here, the notation $\mathbf{W}\boldsymbol{\Lambda}$ corresponds to the GHRR hypervector decomposition described in Section III and Proposition 1. Specifically:
\begin{itemize}
    \item $\mathbf{W}$ (analogous to the unitary matrix basis components in earlier sections) represents the block-diagonal unitary matrix component. In our framework, this matrix can be fixed (randomly generated) or trainable, serving a role similar to the weight matrices in standard attention by orienting the hypervector in the hyperspace.
    \item $\boldsymbol{\Lambda}$ represents the diagonal matrix of phasors (e.g., $e^{i\theta}$) which encodes the actual feature values or token information via fractional power encoding.
\end{itemize}
Then we can write the attention mechanism in GHRR form as:
\begin{align}\label{eq:ghrr-attn}
    [\sigma(\mathbf{Q}*\mathbf{K}^\dagger)*\mathbf{V}]_j = \sigma(\mathbf{W}_{qj}\boldsymbol{\Lambda}_{qj}\boldsymbol{\Lambda}_{kj}^\dagger \mathbf{W}_{kj}^\dagger)\mathbf{W}_{vj}\boldsymbol{\Lambda}_{vj}
\end{align}
where $\sigma(\cdot) := \text{softmax}(\text{Re}[\cdot])$. The similarity between Eq.~\ref{eq:attn2} and Eq.~\ref{eq:ghrr-attn} suggests that GHRR is capable of implementing attention. The GHRR representation enforces unitarity on $\mathbf{W}_j$ to generate base hypervectors that are norm-preserving; we can relax this constraint for greater flexibility, enabling expressivity comparable to that of a traditional Transformer. The softmax is used to ensure attention weights are normalized; other HDC works have utilized the softmax operation in the context of training, so it can be included as part of the framework \cite{hypermetric,herscheConstrainedFewshotClassincremental2022}.

While traditional attention is applied to a sequence of tokens as encoded by a matrix $\mathbf{X}$, here, this is not necessarily the case. The analog of $\mathbf{X}$, $\boldsymbol{\Lambda}$, is a diagonal matrix that in general encodes only one token. Thus, the attention described in Eq.~\ref{eq:ghrr-attn}, while similar in form to Eq.~\ref{eq:attn2}, does not implement attention over tokens; instead, it applies attention over the representation of a single token. We distinguish this form of \textit{representation-level attention} from traditional \textit{token-level attention} which is explicitly applied only to token representations.

\paragraph{Token-level attention and beyond via VSA positional encoding} To add token-level information, one can express the hypervectors as a sum of token hypervectors bound with hypervectors encoding positional information. To simplify notation, let us consider only one dimension of the hypervector, allowing us to omit and free up the subscript $j$ on $\mathbf{W}$ and $\boldsymbol{\Lambda}$ for GHRR components. More precisely, let $\phi(x)=\mathbf{W}\boldsymbol{\Lambda}(x)$ be the GHRR encoding and $\mathbf{E}_j$ be the positional information for the $j$-th token. In particular, we let $\mathbf{E}_j$ be an $m\times m$ matrix such that $[\mathbf{E}_j]_{jj}=1$ and zero everywhere else. Let $x_1,\dots,x_m$ be the tokens we wish to encode. Then our sequence encoding is
\begin{align}\label{eq:pos-enc}
    \phi(x_1,\dots,x_m)&:=\sum_{j=1}^m \mathbf{E}_j\phi(x_j),
\end{align}
which results in a matrix where exactly the $j$-th row encodes information about the $j$-th token. This construction gives an explicit correspondence between GHRR and attention, with the difference being that GHRR token representations involve an extra random exponential map as determined by $\boldsymbol{\Lambda}$.

In addition, we may let $\mathbf{E}_j$ be some arbitrary trainable matrix $\mathbf{P}_j$, which essentially makes it a learnable position encoding present in each transformer block. In contrast to the sinusoidal positional encoding commonly used in the vanilla Transformer architecture, which is added to the token embeddings prior to the application of any Transformer blocks, this \textit{binding-based positional encoding} is applied via the binding operation, i.e. matrix multiplication in GHRR, and is included in every GHRR Transformer block.


\paragraph{Matching parameters between GHRR and Transformer} 

Given the above discussion, we propose to \textit{treat each component of a GHRR hypervector as an attention head}. In other words, the hyperdimension $D$ in GHRR corresponds to the number of attention heads, and the complexity $m$ determines the maximum number of tokens the GHRR implementation of attention can deal with independently. Note that it's possible to encode more than $m$ tokens at the cost of them entangling within the representation. For simplicity in analysis, we do not consider the entangled case for this work; as a result, the range of attention is limited by $m$. 

Beyond just implementing attention, what this suggests is that binding is more expressive in GHRR compared to other VSAs as binding by itself implements attention in GHRR but not in other VSAs. This creates the possibility of interpreting binding beyond simple association.

\subsection{Time Complexity}

We compare the time complexity of binding FHRR and GHRR hypervectors. For an FHRR hypervector of dimension $D'$, the time complexity of the binding operation is $O(D')$. or an FHRR hypervector of dimension $D\times m\times m$, binding involves multiplying $D$ $m\times m$ matrices, which has time complexity $O(Dm^3)$. To compare the two, we set the dimension of the FHRR hypervector to be equal to the total dimension of the GHRR hypervector, i.e. $D'=Dm^2$, resulting in time complexity $O(Dm^2)$ and $O(Dm^3)$ for FHRR and GHRR hypervectors respectively. Thus, the time complexity of GHRR has an additional factor of $m$, which we trade-off for greater expressivity in the binding operation.


\section{Experiments} \label{sec:Exp}

In this section, we validate the theoretical formulation of GHRR proposed in Sections \ref{sec:ghrr-overview} and \ref{sec:ghrr-implementation} through a series of empirical experiments. Our goal is to demonstrate that the GHRR framework not only adheres to the fundamental properties of HDC but also offers superior expressivity and performance in handling complex data structures. Specifically, we design our experiments to verify the following key properties discussed in the theoretical formulation:

\begin{itemize}
    \item \textbf{Flexible Non-commutativity (Sections \ref{sec:non-commutativty} - \ref{sec:q-effect}):} We empirically validate that GHRR binding is non-commutative and that this property can be modulated by the matrix component $Q$. This supports the claim in Section \ref{sec:ghrr-overview} that GHRR can naturally encode order-dependent and context-sensitive structures without requiring external permutation operations.
    
    \item \textbf{Decoding Accuracy for Compositional Structures (Section \ref{sec:decode}):} We test the model's ability to decode elements from complex nested structures (e.g., trees). 
    
    \item \textbf{Binding Capacity (Section \ref{subsec:capacity}):} We evaluate the storage capacity of bound hypervectors to verify that the increased expressivity of GHRR does not compromise memorization capabilities inherent to HDC.
    
    \item \textbf{Attention Mechanism (Section \ref{subsec:token}):} Finally, we demonstrate the practical utility of the GHRR binding operation by replacing the attention mechanism in a Transformer. This validates the theoretical insight in Section \ref{subsec:binding-attention} that GHRR binding is mathematically expressive enough to approximate attention mechanisms in deep learning.
\end{itemize}

\subsection{A demonstration of non-commutativity}\label{sec:non-commutativty}

To demonstrate the non-commutativity of GHRR, we use a simple dictionary as an example. A dictionary associates key-value pairs $\{(k_i, v_i)\}_{i=1}^n$. The corresponding hypervector that encodes the entire dictionary is $\mathbf{H}=\sum_{i=1}^n \mathbf{K}_i * \mathbf{V}_i$, where we bold the keys and values to indicate that they are hypervectors. We require that $\mathbf{K}_i,\mathbf{K}_j$ for $i\neq j$ to be quasi-orthogonal. To retrieve $\mathbf{V}_j$, we compute $\mathbf{K}_j^{-1}*\mathbf{H} = \sum_{i=1}^n \mathbf{K}_j^{-1}\mathbf{K}_i * \mathbf{V}_i
    = \mathbf{V}_j + \mathrm{noise}$.

Figure \ref{fig:nested-structure} compares the decoded hypervectors in a nested dictionary for commutative (FHRR) and non-commutative (GHRR) encodings. The dictionary is encoded via $\mathbf{H}=\mathbf{K}_1*(\mathbf{K}_1*\mathbf{V}_1+\mathbf{K}_2*\mathbf{V}_2)+\mathbf{K}_2*(\mathbf{K}_1*\mathbf{V}_3+\mathbf{K}_2*\mathbf{V}_4)$, values are decoded by $\mathbf{V}_1'=\mathbf{K}_1^{-1}*\mathbf{K}_1^{-1}*\mathbf{H}$, $\mathbf{V}_2'=\mathbf{K}_2^{-1}*\mathbf{K}_1^{-1}*\mathbf{H}$, etc. We observe that the commutative encoding is confused for values with equivalent keys up to permutation, while the non-commutative encoding does not have this issue. This confusion is evidenced by the overlapping lines for the decoded $V_2'$ and $V_3'$ in Figure \ref{fig:nested-structure}B. Meanwhile, these two lines are distinct in Figure \ref{fig:nested-structure}C.

\begin{figure*}
    \centering
    \includegraphics[width=0.85\textwidth]{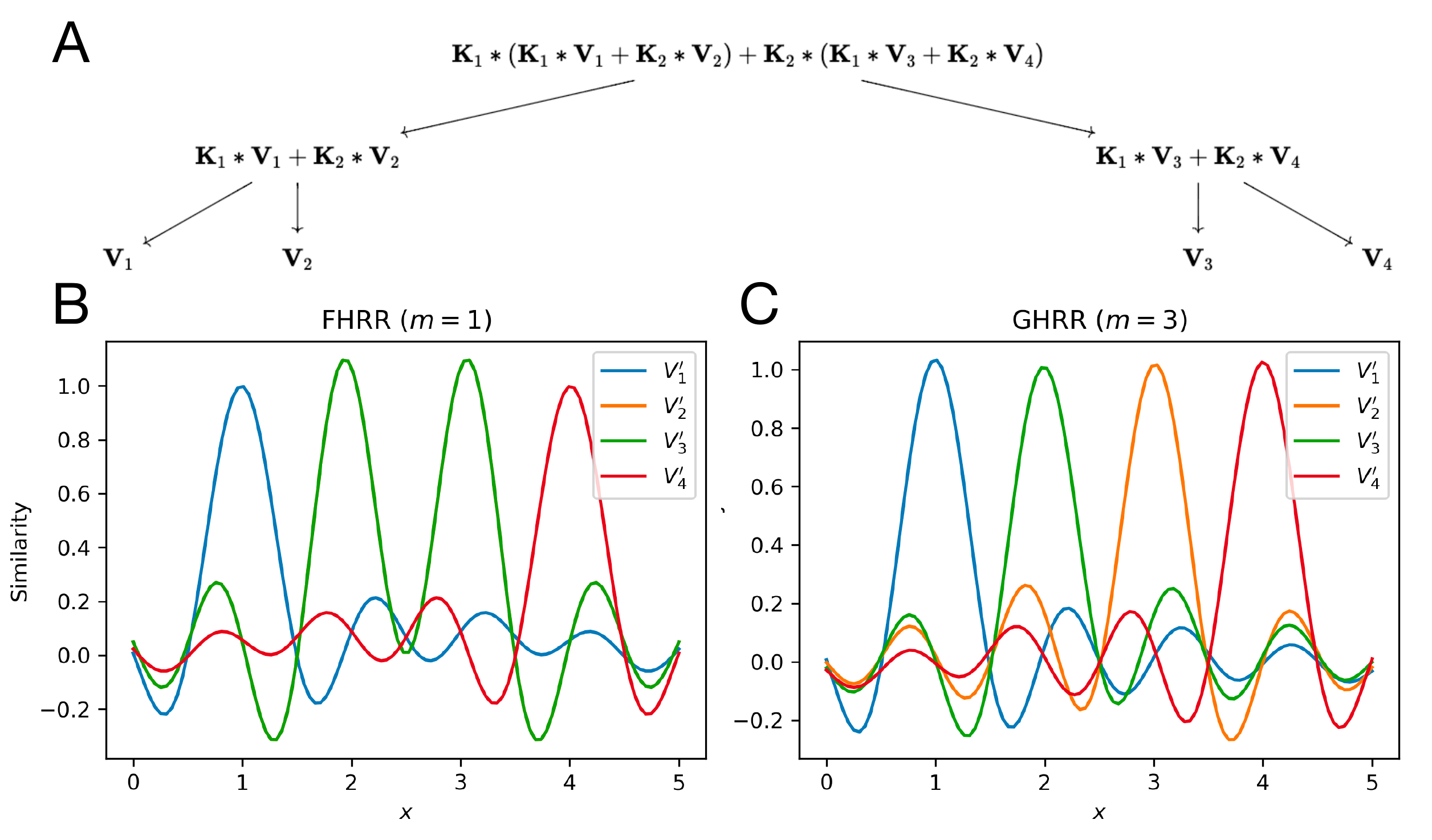}
    \caption{We encode a hypervector $\mathbf{H}=\mathbf{K}_1*(\mathbf{K}_1*\mathbf{V}_1+\mathbf{K}_2*\mathbf{V}_2)+\mathbf{K}_2*(\mathbf{K}_1*\mathbf{V}_3+\mathbf{K}_2*\mathbf{V}_4)$, and retrieve approximate hypervectors $\mathbf{V}_1'=\mathbf{K}_1^{-1}*\mathbf{K}_1^{-1}*\mathbf{H}$, $\mathbf{V}_2'=\mathbf{K}_2^{-1}*\mathbf{K}_1^{-1}*\mathbf{H}$, etc. We then plot the similarities of the retrieved hypervectors against $\phi(x)$. A decoding is successful if there is a peak at the given value and nowhere else. \textbf{A.} A visualization of the encoded structure. \textbf{B.} We use an FHRR encoding, i.e. a commutative encoding. \textbf{C.} We use a GHRR encoding with $m=3$, i.e. a non-commutative encoding.}
    \label{fig:nested-structure}
\end{figure*}

\subsection{Effect of Q on commutativity}\label{sec:q-effect}
To measure the effect of $\mathbf{Q}$ on the commutativity of GHRR representations, we sample random GHRR hypervectors $\mathbf{H}_1,\mathbf{H}_2$ and compute the similarity between $\mathbf{H}_1*\mathbf{H}_2$ and $\mathbf{H}_2*\mathbf{H}_1$. We call the similarity $\delta(\mathbf{H}_1*\mathbf{H}_2,\mathbf{H}_2*\mathbf{H}_1)$ the \textit{degree of commutativity} of the representation. We define the \textit{diagonality} of a matrix $\mathbf{Q}$ to be $\sum_{j} |\mathbf{Q}_{jj}|/\sum_{j}\sum_{k} |\mathbf{Q}_{jk}|$. Let $\mathbf{Q}_1,\mathbf{Q}_2$ correspond to $\mathbf{H}_1,\mathbf{H}_2$ respectively. 

Figure \ref{fig:diagonality} plots the sum of the diagonality of $\mathbf{Q}_1$ and $\mathbf{Q}_2$ against the degree of commutativity of $\mathbf{H}_1$ and $\mathbf{H}_2$. Each point in the figure corresponds to a pair of randomly sampled hypervectors $\mathbf{H}_1,\mathbf{H}_2$ with corresponding matrices $\mathbf{Q}_1,\mathbf{Q}_2$. Each $\mathbf{Q}^{(j)}$ is randomly sampled by first sampling a matrix $\mathbf{X}\in\mathbb{C}^{m\times m}$, where each component is of the form $a+bi$ with $a,b\sim N(0,1)$. We then define $\mathbf{H}(\mathbf{X})=(\mathbf{X}+\mathbf{X}^\dagger)/2$ and $\mathbf{Q}_j(X)=\exp(i\mathbf{H}(\mathbf{X}))$. We perform gradient descent on $\mathbf{X}$ such that $\mathbf{Q}_j(\mathbf{X})$ as diagonality close to the desired value.

We observe that the two quantities are strongly correlated. Moreover, for larger $m$, there is a high probability of sampling a matrix with low diagonality and thus low degree of commutativity.

\begin{figure}
    \centering
    \includegraphics[width=\columnwidth]{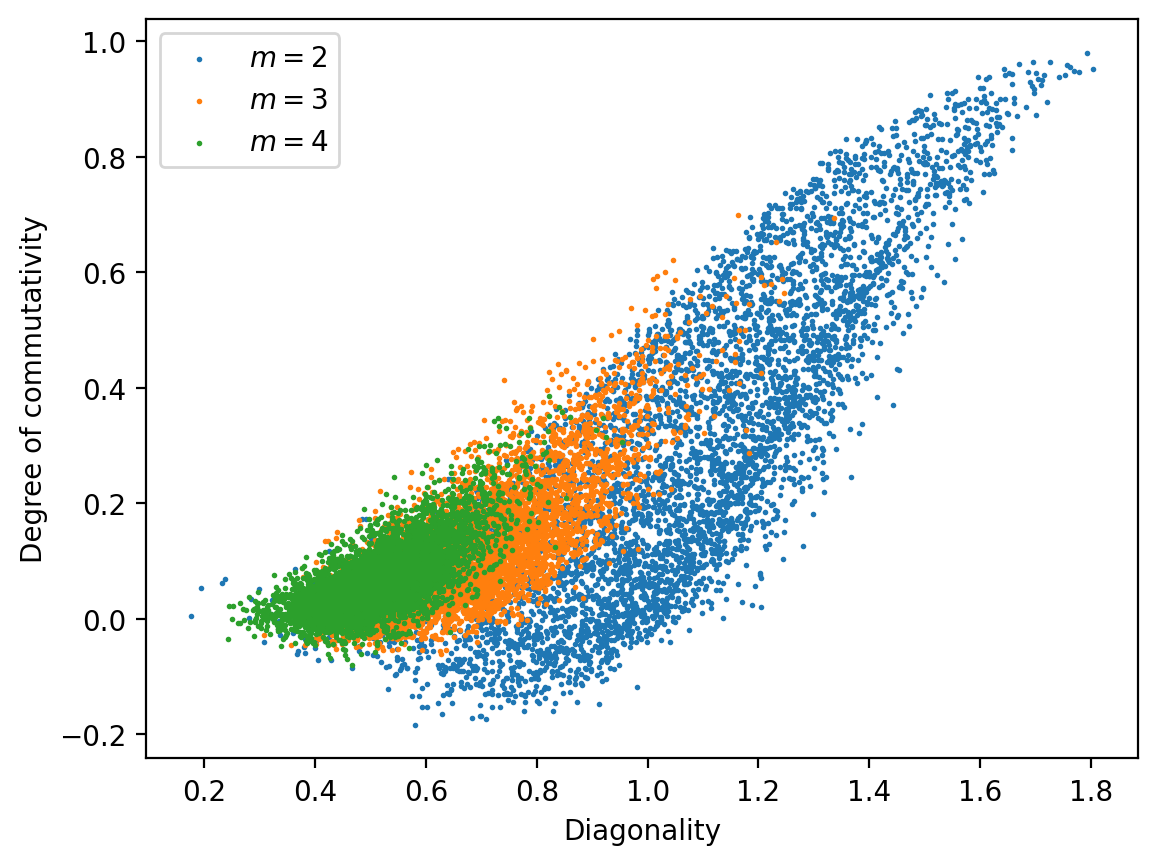}
    \caption{Plot of the sum of the diagonality of $\mathbf{Q}_1$ and $\mathbf{Q}_2$ against the degree of commutativity of $H_1$ and $H_2$.}
    \label{fig:diagonality}
\end{figure}

\subsection{Decoding accuracy for nested structures}\label{sec:decode}
To investigate the effect of $m$, we use a similar setup to the one Section \ref{sec:non-commutativty} but vary the depth of the encoded dictionary. A value is decoded successfully if it has the highest similarity to the decoded hypervector compared to all other values. We plot the decoding accuracy over different depths and values of $m$ in Figure \ref{fig:decode-acc}A, using a total dimension $Dm^2$ of 600. We observe that the larger $m$ is, the longer the high decoding accuracy is sustained for increasing tree depths. However, the overall decoding accuracy also drops faster for larger $m$ after a given tree depth.

In Figure \ref{fig:decode-acc}B, we use the same procedure as in Figure \ref{fig:decode-acc}A, but apply the permutation operation $\rho$ to each subtree encoding. Given a hypervector of dimension $D$, the permutation operation is typically implemented by rotating its vector entries; e.g. the first element becomes the last, while the 2nd to $D$-th elements are shifted upwards. An example of a permutation-based representation of a tree would be $\mathbf{H}=\mathbf{K}_1*\rho(\mathbf{K}_1*\rho\mathbf{V}_1+\mathbf{K}_2*\rho\mathbf{V}_2)+\mathbf{K}_2*\rho(\mathbf{K}_1*\rho\mathbf{V}_3+\mathbf{K}_2*\rho\mathbf{V}_4)$.

Interestingly, we observe that the decoding accuracy is sustained at 100\% for longer as we increase the depth of the tree, but decreases sharply after a certain threshold. This is in contrast to Figure \ref{fig:decode-acc}A, where the decoding accuracy drops below 100\% earlier, but decays much more gracefully for larger $m$. The results suggest that using GHRR to encode data structures can provide a simpler implementation (e.g. without permutation) whose decoding accuracy degrades gracefully as the size of the data structure saturates the representation.

\begin{figure*}
    \centering
    \includegraphics[width=0.7\textwidth]{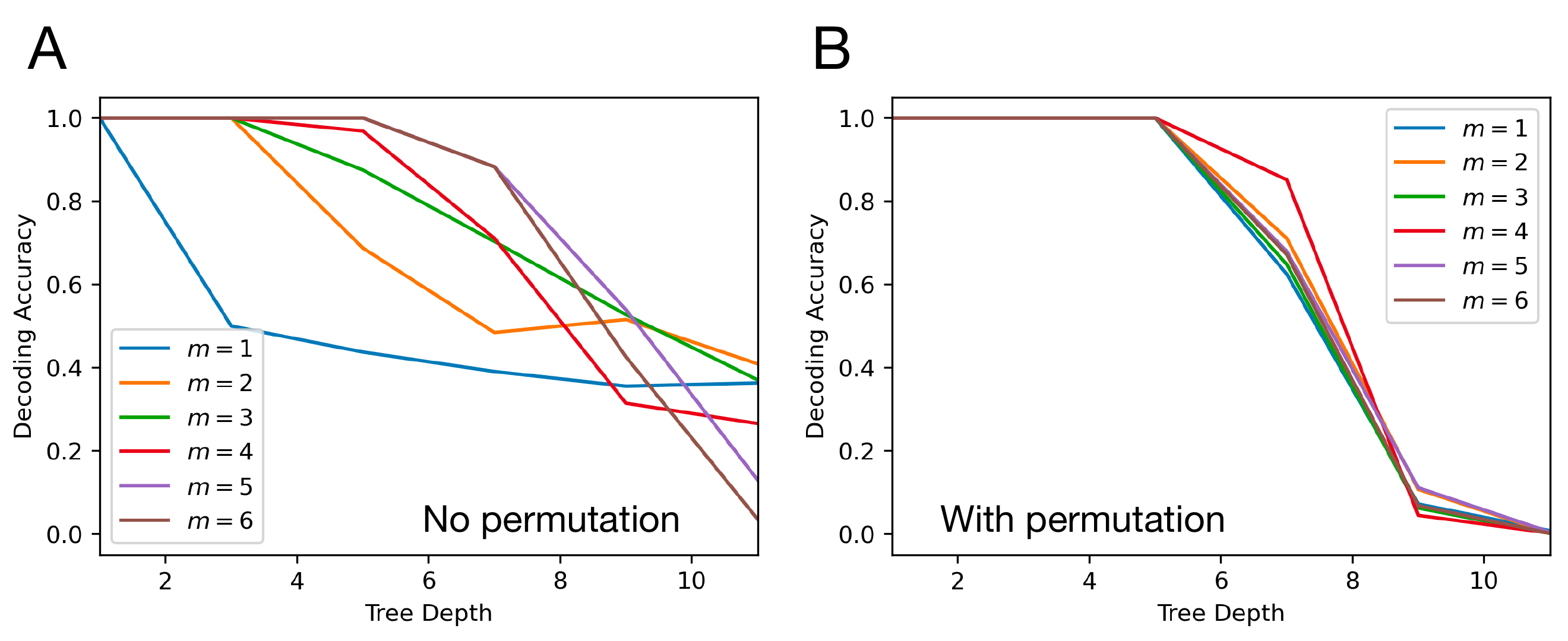}
    \caption{Plot of decoding accuracy for trees of different depths. \textbf{A.} We encode a dictionary as in Figure \ref{fig:nested-structure} but with varying depth and GHRR parameter $m$. We use a total dimension of 600 and plot the resulting decoding accuracy. \textbf{B.} Same encoding as in (A) but with permutation applied to each subtree encoding.}
    \label{fig:decode-acc}
\end{figure*}

We hypothesize that this is due to the exclusivity of the permutation operation; i.e. permutation maps a hypervector to a quasi-orthogonal hypervector, making the decoding quality better. However, this also leads to faster saturation of the tree encoding, leading to a sharp drop in decoding accuracy. On the other hand, the GHRR encoding by itself has a more adjustable range of exclusivity, which, as we will see, is correlated to diagonality, leading to the hypervector having greater memorization capacity \cite{fradyTheorySequenceIndexing2018} albeit for hypervectors that are quasi-orthogonal to a lesser degree.

In addition, we investigate the effect of diagonality on decoding accuracy. Figure \ref{fig:decode-diag} plots the decoding accuracy for different tree depths at different levels of diagonality and different values of $m$. \ref{fig:decode-diag}A, \ref{fig:decode-diag}B, \ref{fig:decode-diag}C, and \ref{fig:decode-diag}D correspond to diagonalities of $0$, $1/3$, $2/3$, and $1$, respectively. We observe that modulating diagonalities from $0$ to $1$ enables us to interpolate between decoding performance similar to that of an FHRR encoding and an FHRR encoding with permutation. Indeed, with a diagonality of $1$, our GHRR encoding becomes commutative, making it functionally equivalent to an FHRR encoding. On the other hand, one can think of a permutation matrix as having diagonality zero (i.e. maximally non-commutative); thus we may expect similar degrees of commutativity for other matrices of diagonality zero. \textit{This result suggests that one way to interpret $\mathbf{Q}$ is as a flexible version of the permutation operation that modulates the level of commutativity of the representation.}

This observation also provides an explanation for the trends seen in Figure \ref{fig:decode-acc}A where larger $m$ leads to a decoding accuracy more similar to that of the permutation encoding. For larger $m$, as suggested by Figure \ref{fig:diagonality}, we are more likely to sample a matrix with low diagonality, which in turn makes the encoding more non-commutative just like permutation encoding.

\begin{figure*}
    \centering
    \includegraphics[width=0.7\textwidth]{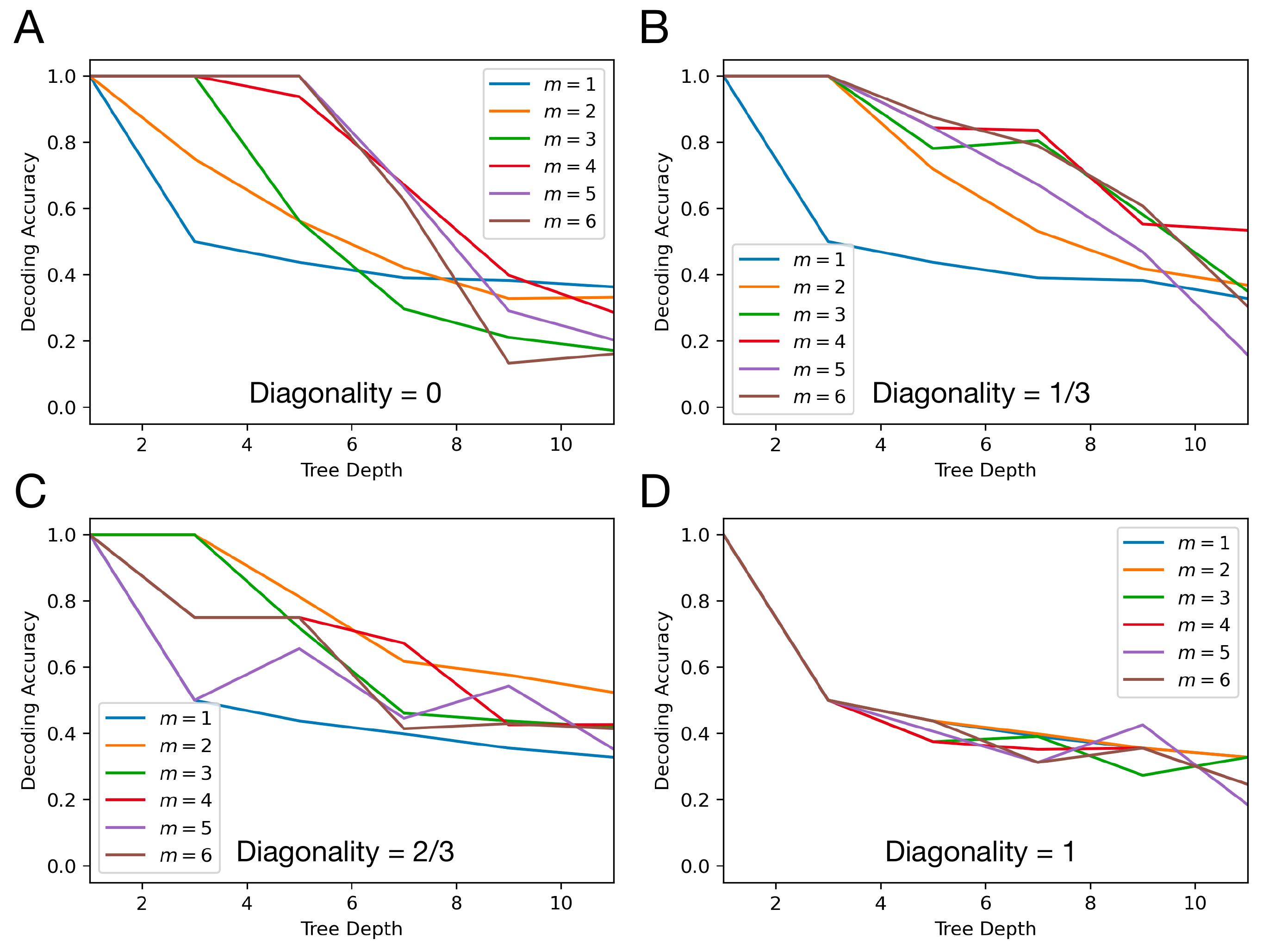}
    \caption{Plot of decoding accuracy for trees of different depths. We encode a dictionary as in Figure \ref{fig:nested-structure} but with varying depth and GHRR parameter $m$. We use a total dimension $Dm^2$ of 600 and plot the resulting decoding accuracy. \textbf{A.} GHRR encoding with diagonality $0$. \textbf{B.} GHRR encoding with diagonality $1/3$. \textbf{C.} GHRR encoding with diagonality $2/3$. \textbf{D.} GHRR encoding with diagonality $1$.}
    \label{fig:decode-diag}
\end{figure*}

\subsection{Capacity of GHRR representations} \label{subsec:capacity}

Capacity refers to the number of hypervectors one can bundle while still being able to memorize it accurately; it measures how efficiently HDC can store (bundle) information into one representation~\cite{fradyTheorySequenceIndexing2018,thomasTheoreticalPerspectiveHyperdimensional2022, rahimi2017high, schlegelComparisonVectorSymbolic2022}. We first investigate the bundling capacity of GHRR base hypervectors following the standard experiment in~\cite{schlegelComparisonVectorSymbolic2022}. We generate 1000 random GHRR hypervectors and randomly select $k$ of them to create a bundle $\mathbf{X}$. We select the top-$k$ hypervectors most similar to $\mathbf{X}$ and compute the decoding accuracy as the proportion of correctly decoded hypervectors within the bundle. We repeat the experiment 10 times and report the mean in Figure \ref{fig:heatmap}. In the figure, we consider multiple variants of GHRR as a result of varying $m=1,...,4$ and choosing varying or fixed $\mathbf{Q}$. Here, $\mathbf{\boldsymbol{\Lambda}}=\mathrm{diag}(e^{i\theta_1},...,e^{i\theta_m})$ where $\theta_j\sim\mathrm{Unif}(0,2\pi)$. We observe negligible differences between the decoding accuracies of different variants, suggesting that the capacity of GHRR is comparable to that of FHRR (GHRR $m=1$).

\begin{figure}
    \centering
    \includegraphics[width=\columnwidth]{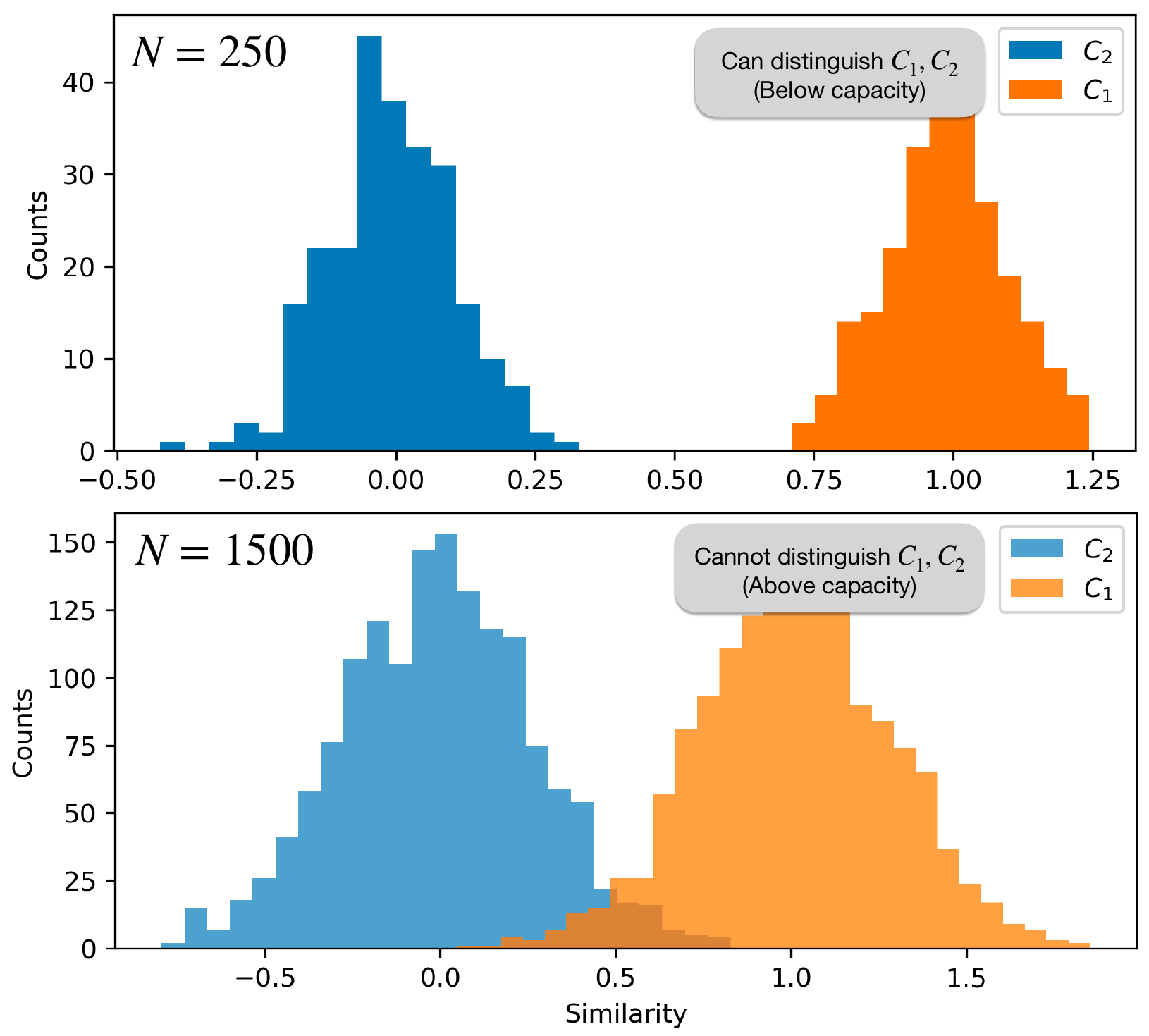}
    \caption{Histogram of similarities between $x$ and $C_1$ and $C_2$ for all $x\in X$. Above: $C_1$ and $C_2$ are a bundles of 250 hypervectors. $\delta(x,C_1)>\delta(x,C_2)$ for all $x\in X$, so the histograms do not overlap. Below: $C_1$ and $C_2$ are a bundles of 1500 hypervectors. There exists $x\in C_1$ such that $\delta(x,C_1)\leq \delta(x,C_2)$, so the histograms overlap.}
    \label{fig:cap-vis}
\end{figure}

\begin{figure}
    \centering
    \includegraphics[width=\columnwidth]{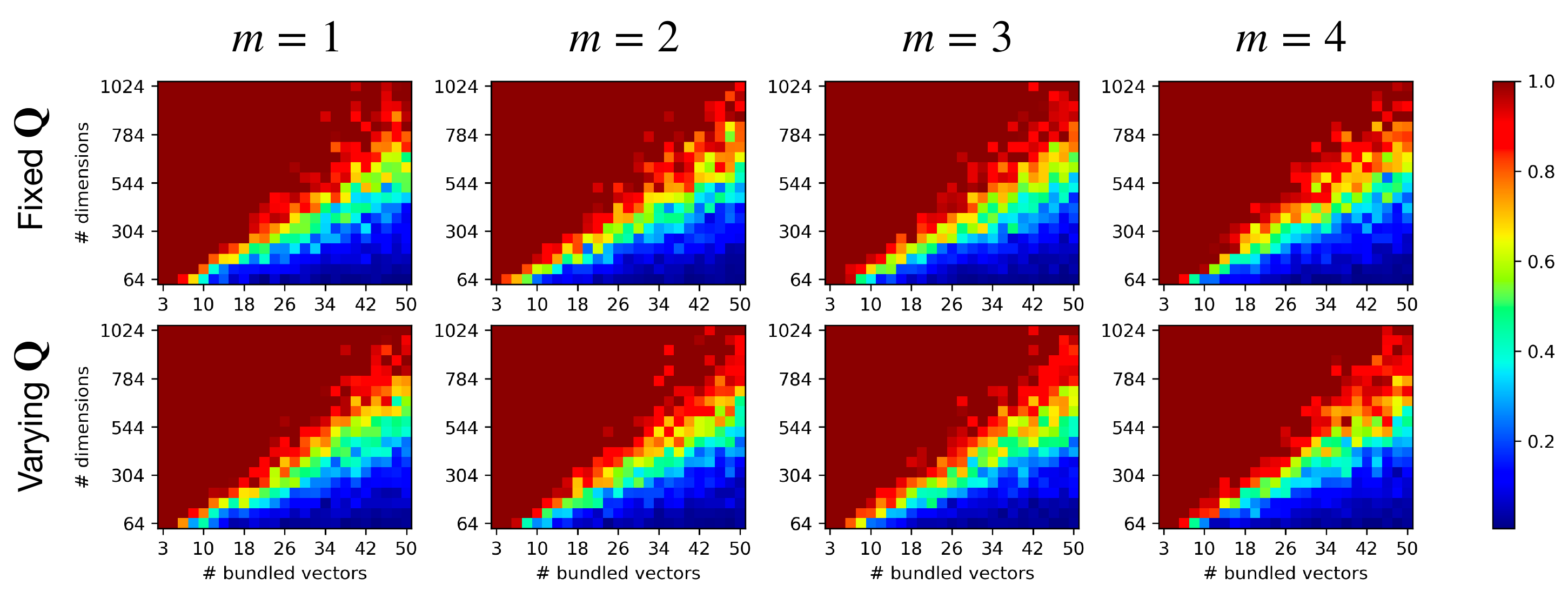}
    \caption{Heatmaps showing the decoding accuracies of different variants of GHRR as we vary total dimension $Dm^2$ and number of bundled vectors $k$.}
    \label{fig:heatmap}
\end{figure}

While bundling capacity is a standard measure used in evaluating HDC implementations, binding is also an inseparable part of the framework: many HDC models operate on composite hypervectors that are the results of the bundling of bound hypervectors~\cite{ osipov2022hyperseed, poduval2022graphd, zou2022biohd}. Taking this factor into account, we are interested in evaluating the effect of binding on capacity for GHRR. 
We design the following experiment to provide a standard comparison across GHRR of different complexities and bound components (how many hypervectors are bound together before bundling the product with the others). 
We consider the case where the set $S$ consists of bound hypervectors whose component base hypervectors are taken from a finite set $U$. Specifically, let $n$ be the number of components in the bound hypervector (e.g. $\mathbf{H}_1*\mathbf{H}_2*\mathbf{H}_3$ has 3 components). $U$ consists of $\sqrt[n]{15000}$ randomly generated quasi-orthogonal base hypervectors. Elements in $S$ are sampled uniformly randomly without replacement from all possible strings of length $n$ with alphabet $U$. We measure capacity by starting with a set $S$ of hypervectors, partitioning it into two sets $X,X'$, and forming bundled hypervectors $\mathbf{C}_1=\sum_{x\in X}x$ and $\mathbf{C}_2=\sum_{x\in X'}x$. We say $X$ is memorized if $\delta(\mathbf{H},\mathbf{C}_1)>\delta(\mathbf{H},\mathbf{C}_2)$ for all $\mathbf{H}\in X$. The capacity is the largest $|X|=:N$ for which this property holds. This can also be done indirectly by plotting the decoding accuracy for various values of $|X|=k$. Figure \ref{fig:cap-vis} visualizes the intuition behind capacity.

Figure \ref{fig:capacity} consists of plots of how capacity changes as total dimension $Dm^2$ increases for different parameters $m$ and the number of bound components. In Figure \ref{fig:capacity}A, permutations of bound hypervectors are considered distinct, while in Figure \ref{fig:capacity}B, permutations are considered the same and are thus removed.

In Figure \ref{fig:capacity}A, for memorizing the usual base hypervectors (i.e. one bound component hypervectors), there is virtually no difference in capacity between different $m$. Capacity increases approximately linearly with respect to the total dimension $Dm^2$. For a larger number of bound components, there is a greater distinction between the capacity for different values of $m$. In particular, $m=1$ (FHRR) is significantly worse. Interestingly, we find that for fixed $m>1$, the difference in capacity for different numbers of bound components is not large. Meanwhile, in Figure \ref{fig:capacity}B, there is a negligible difference between the capacity for different parameters $m$; for all $m$, the capacity scales linearly with respect to the total dimension. Taken together, this suggests that the reason FHRR ($m=1$) performs significantly worse in \ref{fig:capacity}A is due to its inability to distinguish permutations of bound hypervectors. Conversely, increasing $m$ makes the GHRR representation non-commutative and thus able to distinguish between the permutations. At the same time, there is a negligible loss in capacity when we increase $m$ while keeping the total dimension $Dm^2$ fixed. \textit{These results suggest that GHRR provides a flexible non-commutative representation with no loss in capacity compared to FHRR.}

Table \ref{tab:cap} provides a quantitative measure of the capacity of bound hypervectors. Here, hypervectors have a total dimension of $Dm^2=900$ and the number of possible bound hypervector combinations is kept to be approximately 15000. Permutations are not considered distinct. We report the mean and standard deviation of the capacities computed over 20 trials. 

{
In practice, the block size $m$ controls a trade-off between expressivity and model complexity. Larger $m$ makes the binding operation more expressive and more strongly non-commutative, which helps decoding deeper or more structured compositions, but also increases the number of parameters and the per-layer computation. Smaller $m$ yields cheaper, more FHRR-like encodings with weaker order sensitivity. Within the range of $m$ we study, we did not observe clear evidence that increasing $m$ alone causes overfitting or loss of robustness, so $m$ can be treated as a tuning knob to balance decoding performance against model size and compute budget.}

\begin{figure}
    \centering
    \includegraphics[width=\columnwidth]{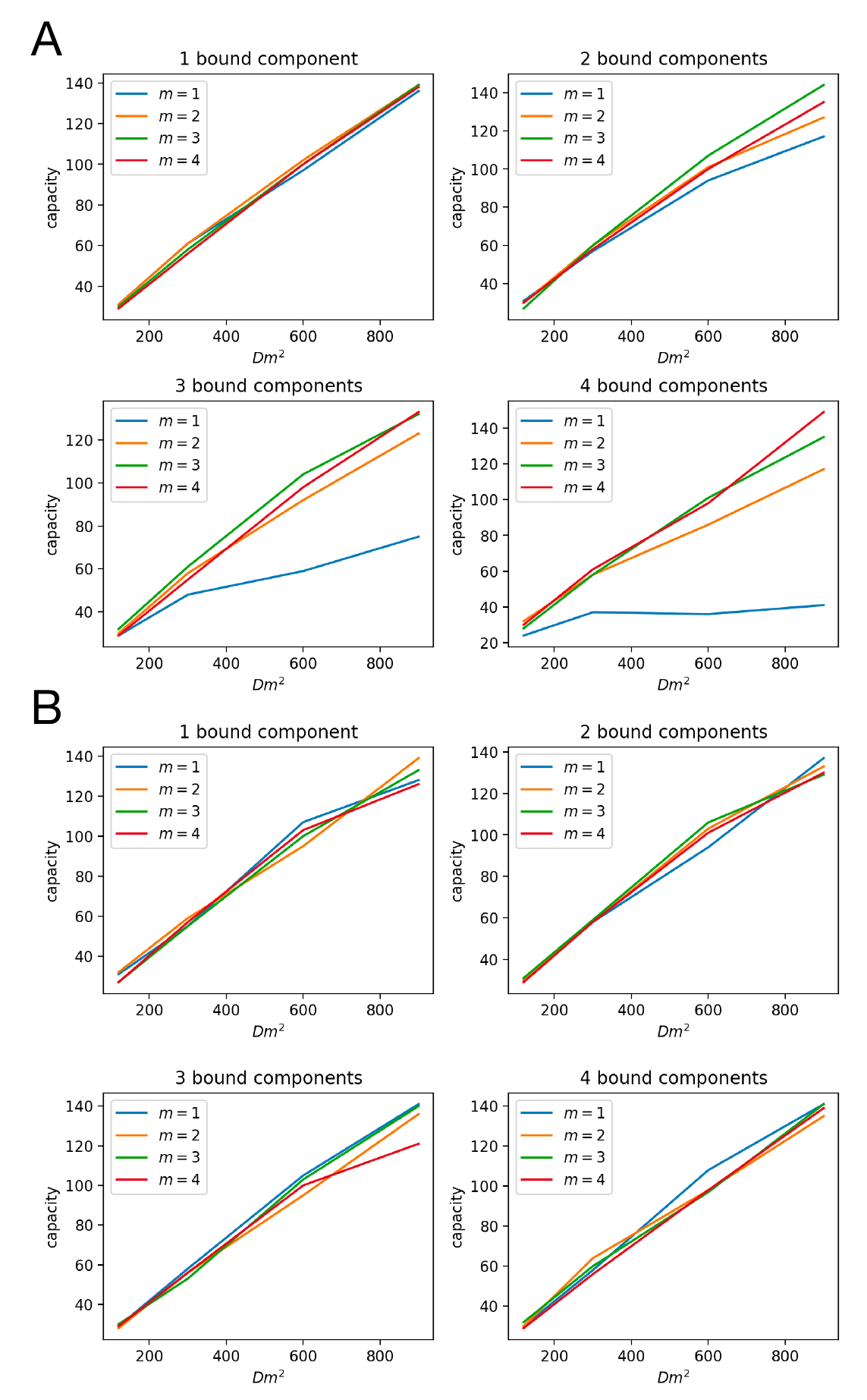}
    \caption{The capacity of bound GHRR representations with varying total dimension $Dm^2$ and the number of bound components. \textbf{A.} Plots where permutations are considered distinct; e.g. $\mathbf{A}*\mathbf{B}\neq \mathbf{B}*\mathbf{A}$ \textbf{B.} Plots where permutations are considered the same (e.g. $\mathbf{A}*\mathbf{B}= \mathbf{B}*\mathbf{A}$) and are thus removed.}
    \label{fig:capacity}
\end{figure}


\begin{table}
\centering
\begin{tabular}{ccccc}
\toprule
Components & $m=1$              & $m=2$              & $m=3$              & $m=4$              \\ \midrule
$1$                                    & $272 \pm 34$ & $279 \pm 52$ & $278 \pm 45$ & $277 \pm 39$ \\
$2$                                    & $234 \pm 34$ & $255 \pm 64$ & $289 \pm 35$ & $271 \pm 31$ \\
$3$                                    & $151 \pm 33$ & $247 \pm 37$ & $265 \pm 52$ & $266 \pm 47$ \\
$4$                                    & $83 \pm 20$  & $234 \pm 35$ & $270 \pm 50$ & $299 \pm 61$ \\ \bottomrule
\end{tabular}
\caption{Capacity of bound hypervectors}
\label{tab:cap}
\end{table}

\subsection{Next Token Prediction} \label{subsec:token}
We replace the attention mechanism of transformers with the GHRR equivalent as described above and test the model on a language modeling task.

Let $x_1,\dots,x_n$ be a sequence. For notational simplicity, assume $D=1$, so all hypervectors are simply $m\times m$ matrices (though note that we use $D>1$ in experiments). Moreover, we assume $m=n$. Let $\phi_q,\phi_k,\phi_v$ be query, key, and value encoders respectively. We define the hypervectors $\mathbf{Q}=\sum_{j=1}^n \mathbf{P}_j^q * \phi_q(x_j)$, $\mathbf{K}=\sum_{j=1}^n \mathbf{P}_j^k * \phi_k(x_j)$, and $\mathbf{V}=\sum_{j=1}^n \mathbf{P}_j^v * \phi_v(x_j)$, where $\mathbf{P}_j^q,\mathbf{P}_j^k,\mathbf{P}_j^v$ for $j=1,\dots,n$ are positional encoding (PE) matrices. With $D=1$, binding reduces to matrix multiplication.

We evaluate our model on a next-token prediction language modeling task on the Wikitext2 \cite{wikitext2} and the Penn Treebank \cite{ptb} datasets. 

The PE can be the same or different across the $\mathbf{Q}, \mathbf{K}, \mathbf{V}$ hypervectors as well as across attention heads. Moreover, the PEs can either be trainable or fixed (i.e. randomly initialized). This gives us eight different variants of the GHRR Transformer model. For each GHRR encoder $\phi_q,\phi_k,\phi_v$, we make $W$ trainable and keep $\boldsymbol{\Lambda}$ fixed. Sample positional encodings are visualized in Supplementary Materials B. 

Both the baseline Transformer model \cite{vaswaniAttentionAllYou2017} and our model have a comparable number of weight parameters, with a slight increase when trainable PE is included. Training details are described in Supplementary Materials B. We report the mean perplexity (PPL) and standard deviation over five independent runs in Table \ref{tab:sequence}. { ``All same" uses a single PE shared across $Q, K$, and $V$ and across all attention heads; ``QKV" uses distinct PEs for $Q, K$, and $V$ but shares them across heads; ``heads" shares PEs across heads $Q, K$, and $V$; ``All different" uses distinct PEs for $Q, K$, and $V$ and for each head. For each configuration, we report the mean and standard deviation over five runs.}

We observe an average performance improvement of 5.47\% on WikiText-2 and 2.75\% on the PTB dataset when compared to the baseline Transformer model. In the cases with the highest observed improvements, the performance increased by 5.53\% on WikiText-2 and 3.44\% on PTB, respectively. Specifically, we found a 2.56\% performance improvement in PTB when PEs are varied across $\mathbf{Q}, \mathbf{K}, \mathbf{V}$ matrices, attention heads, or both, compared to when they were not. 

This suggests that the inclusion of PEs help with model performance, though there needs to be some kind of variation between PEs to have sufficient expressive power. Moreover, there is negligible difference between fixed and trainable positional encodings, suggesting that level of expressive power is not needed for this task.

\begin{table}
    \centering
    \begin{tabular}{cccc}
        \toprule
        Model & Trainable PE & Wikitext2 \cite{wikitext2} & Penn Treebank \cite{ptb} \\ \midrule
        Baseline & N/A & 29.16 $\pm$ 0.13 & 94.78 $\pm$ 0.23 \\ \midrule
        All same & No & \textbf{27.54 $\pm$ 0.08} & 93.93 $\pm$ 0.09 \\
        QKV & No & 27.6 $\pm$ 0.10 & 91.63 $\pm$ 0.04 \\
        Head & No & 27.55 $\pm$ 0.06 & 91.68 $\pm$ 0.11 \\
        All different & No & 27.56 $\pm$ 0.09 & 91.58 $\pm$ 0.08 \\ \midrule
        All same & Yes & \textbf{27.54 $\pm$ 0.08} & 93.98 $\pm$ 0.12 \\
        QKV & Yes & 27.57 $\pm$ 0.12 & \textbf{91.51 $\pm$ 0.08} \\
        Head & Yes & 27.56 $\pm$ 0.07 & 91.53 $\pm$ 0.09 \\
        All different & Yes & 27.56 $\pm$ 0.05 & 91.52 $\pm$ 0.05 \\ \bottomrule
    \end{tabular}
    \caption{Perplexity of trained language models with different positional encoding (PE) sharing strategies.}
    \label{tab:sequence}
\end{table}

\section{Discussion}
In Section \ref{sec:enc-data}, we outlined several variations of GHRR with increasing complexity, controlled by whether (1) $\mathbf{Q}$ varies over the dimensions of the GHRR hypervector; and (2) whether $\mathbf{Q}$ depends on the input $\mathbf{x}$. We subsequently explored the simplest variation where $\mathbf{Q}$ is constant and briefly discussed the kernel properties for an encoding with varying $\mathbf{Q}$ over the dimensions. In this section, we discuss the other variations and their possibilities for future development.

\subsection{Varying Q over dimensions}
Varying $\mathbf{Q}$ over the dimensions will not substantially affect the shape of the kernel. As we sample $\mathbf{Q}$ and $\mathbf{\boldsymbol{\Lambda}}$ independently of each other, expectations can be factored as in Eq.~\ref{eq:tr-exp-val}, which suggests that it is enough to know the expected value of $\mathbf{Q}$ to know the shape of the kernel. However, as noted previously, in an encoding scheme that is purely randomly sampled, having varying $\mathbf{Q}$ provides stable behavior as when $D$ increases, the behavior converges to the mean. In addition, having varying $\mathbf{Q}$ over dimensions, as suggested by Eq.~\ref{eq:elem-bind}, places different emphases on the conjunction of hypervectors when binding them together. It is not clear how this affects the characteristics of binding, but it would be an interesting line of investigation.

\subsection{Input-dependent Q}
As highlighted in Eq.~\ref{eq:weighted-kernel}, $\mathbf{Q}$ can be thought of as a re-weighting of the kernels corresponding to the elements of the diagonal in $\mathbf{\boldsymbol{\Lambda}}$. Thus, by having $\mathbf{Q}$ depend on the input, we can modulate how important each of the kernels is depending on the context for different pairs of inputs. For example, suppose the kernels $K_k$ in the sum correspond to Gaussian kernels with different length scales. Then, for inputs $\mathbf{x},\mathbf{y}\in \mathcal{X}$, it is possible to learn a map $\mathbf{Q}:\mathcal{X}\to\mathcal{H}$ that places emphasis on different kernels depending on the diagonal of the matrix $\mathbf{Q}(\mathbf{x})\mathbf{Q}(\mathbf{y})^\dagger$, which enables us to compare items at different scales depending on context.

Taken together, the above two parameters enable increased complexity and flexibility of GHRR, affecting both the characteristics of binding and the shape and adaptivity of the kernel. It remains to be seen, however, what particular implementations of the map $\mathbf{Q}:\mathcal{X}\to\mathcal{H}$ would work best for various purposes.

{
\subsection{Extensions to multimodal and control tasks}
The kernel view of GHRR suggests several natural extensions beyond the tasks considered in this work. First, multimodal inputs can be encoded by assigning different $\Lambda$-distributions and Q-components to different modalities (e.g., vision, language) and bundling them into a joint GHRR hypervector. The resulting similarity function is then a mixture of modality-specific kernels, enabling flexible cross-modal interactions while maintaining fixed-width representations. Second, continuous control states and actions can be encoded via fractional power encoding, with GHRR binding used to represent state--action pairs and their compositions. This would allow policies or value functions to be expressed in a hyperdimensional space with holographic structure, potentially benefiting from the robustness and compositionality of GHRR. A detailed investigation of these multimodal and control applications is beyond the scope of the present paper and is left for future work.
}

{
    \subsection{Alternative non-commutative HDC implementations}
    There exist other models that have non-commutative binding, including HRR with circular convolution~\cite{plate1995holographic}, VTB~\cite{gosmann2019vector}, MBAT~\cite{gallant2013representing}, and SMR~\cite{kelly2013encoding, kleykoSurveyHyperdimensionalComputing2023}. Many of them share similar considerations, but their solution typically breaks more ``rules" compared to our methods. In particular, the binding operators in HRR and VTB are not associative. Although there is an interesting argument to justify non-associativity in terms of modeling hierarchy~\cite{gosmann2019vector}, we did not consider this for our design as we do not think it is a useful feature in terms of integration into neural networks. MBAT has a slightly different view in binding: it uses matrices as keys and performs binding with matrix-vector multiplication of key and value followed by bundling the key-value pairs. Their binding thus consists of two steps - matrix-vector multiplication and bundling - which is relatively uncommon. SMR uses one large square matrix instead of a list of square matrices as in GHRR to represent data, which limits its scalability.
}

\section{Conclusion}
In this work, we introduced the GHRR framework, an extension of FHRR, and provided a particular implementation of the framework. We proved that GHRR maintains the theoretical properties of traditional HDC representations and provide empirical demonstrations of quasi-orthogonality. We explored the kernel and binding properties of GHRR and provided an interpretation of binding in GHRR as an interpolation between binding in FHRR and in Tensor Product Representations. We performed empirical experiments on GHRR, demonstrating its flexible non-commutativity, increased decoding accuracy for compositional structures, and improved memorization capacity for bound hypervectors compared to FHRR.

We also demonstrate that binding in GHRR is more expressive than that in other HDC variants; in particular, we show that binding in GHRR can implement a kind of attention mechanism. We verify this by replacing the attention mechanism in a transformer with its GHRR-equivalent and testing it on a language modeling task, showing improved performance compared to a vanilla transformer.

\section{Author's Contributions}
Calvin Yeung and Zhuowen Zou conceptualized the project, developed the theory, and drafted the manuscript. Calvin Yeung, SungHeon Jeong, Wenjun Huang conducted the experiments, performed the analysis, and data visualization. Nathaniel Bastian and Mohsen Imani provided guidance and feedback. All authors reviewed the manuscript.

\section{Acknowledgements}
This work was supported in part by DARPA Young Faculty Award, National Science Foundation \#2127780, \#2319198,  \#2321840 and \#2312517, Semiconductor Research Corporation (SRC), Office of Naval Research Young Investigator Program Award, grants \#N00014-21-1-2225 and \#N00014-22-1-2067, the Air Force Office of Scientific Research under award \#FA9550-22-1-0253, the U.S. Military Academy under Cooperative Agreement No. W911NF-24-2-0200, generous gifts from Xilinx and Cisco. The views and conclusions expressed in this paper are those of the authors and do not reflect the official policy or position of the U.S. Military Academy, U.S. Army, U.S. Department of War, or U.S. Government.

\bibliographystyle{ieeetr}
\bibliography{references}

\begin{IEEEbiography}[{\includegraphics[width=1in,height=1.0in,clip,keepaspectratio]{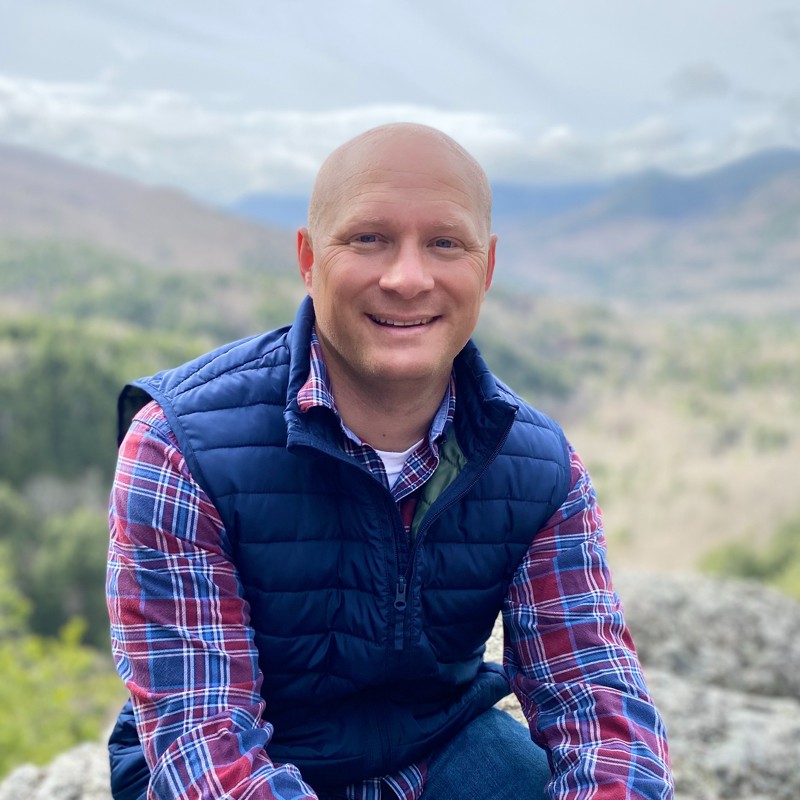}}]{Nathaniel D. Bastian} 
(Senior Member, IEEE) received the Ph.D. degree in Industrial Engineering and Operations Research from Pennsylvania State University, University Park, PA, in 2016. He is currently an Assistant Professor within the Department of Electrical Engineering \& Computer Science at the United States Military Academy at West Point, and he serves as Deputy Director of the Robotics Research Center and Principal Investigator of the Laboratory for Artificial Intelligence Research \& Engineering (LAIRE). His primary research interests combine mathematical optimization, decision theory, machine learning, and statistical computing to design and develop secure, robust, and resilient neurosymbolic AI-enabled systems. He has received \$8M+ in research funding support from DARPA, NSA, OUSW, DEVCOM, AFRL, ONR, and more.
\end{IEEEbiography}

\end{document}